\documentclass[11pt]{article}

\usepackage{authblk}
\usepackage[title]{appendix}

\usepackage{fullpage}
\usepackage[small,bf]{caption}
\usepackage[top=1in,bottom=1in,left=1in,right=1in]{geometry}
\usepackage{fancybox}

\usepackage{mathrsfs}
\usepackage{amsmath, amsfonts, dsfont, amssymb, stmaryrd, amsthm}
\usepackage{algorithm}
\usepackage{url}  
\usepackage{verbatim}
\usepackage{color}
\usepackage{graphicx}  
\usepackage{algorithm}
\usepackage{algorithmic}
\newsavebox{\algbox}
\usepackage{mathtools}
\usepackage{caption}
\usepackage{subcaption}
\usepackage{scalerel,stackengine}
\usepackage{wrapfig}
\usepackage{titlesec}
\stackMath
\newcommand\reallywidehat[1]{%
\savestack{\tmpbox}{\stretchto{%
  \scaleto{%
    \scalerel*[\widthof{\ensuremath{#1}}]{\kern-.6pt\bigwedge\kern-.6pt}%
    {\rule[-\textheight/2]{1ex}{\textheight}}
  }{\textheight}%
}{0.5ex}}%
\stackon[1pt]{#1}{\tmpbox}%
}

\newtheorem{conj}{Conjecture}
\newtheorem{thm}{Theorem}
\newtheorem{rmk}[conj]{Remark}

\newtheorem{lem}[conj]{Lemma}
\newtheorem{prop}[conj]{Proposition}

\newtheorem{defn}[conj]{Definition}
\newtheorem{coro}{Corollary}
\newtheorem{assumption}{Assumption}
\newtheorem{prob}{Problem}

\newcommand{\f}[1]{\boldsymbol{#1}}
\newcommand{\bb}[1]{\mathbb{#1}}
\newcommand{\ca}[1]{\mathcal{#1}}
\newcommand{\m}[1]{\mathrm{#1}}
\newcommand{\s}[1]{\mathsf{#1}}
\newcommand{\ie}{\emph{i.e.}}
\newcommand{\RUFF}{\s{RUFF}}
\newcommand{\PUFF}{\s{PUFF}}
\newcommand{\UFF}{\s{UFF}}
\newcommand{\CFF}{\s{CFF}}
\newcommand{\coord}{\s{rep}}
\newcommand{\response}{\s{response}}
\newcommand{\alignment}{\s{align}}

\newcommand{\remove}[1]{}

\newcommand{\1}{\texttt{1}} 
 
\title{Recovery of sparse linear classifiers from mixture of responses}

\author{Venkata Gandikota~\thanks{V. Gandikota is with the Electrical Engineering \& Computer Science Department at the Syracuse University, NY 13210, USA (email: \texttt{gandikota.venkata@gmail.com}).} ~~~Arya~Mazumdar\thanks{A. Mazumdar is with the Computer Science Department at the University of Massachusetts Amherst, Amherst, MA 01003, USA (email: \texttt{arya@cs.umass.edu}).} ~~~Soumyabrata~Pal\thanks{S. Pal is with the Computer Science Department at the University of Massachusetts Amherst, Amherst, MA 01003, USA (email: \texttt{spal@cs.umass.edu}).}}

\begin{document}
\maketitle

\begin{abstract}
In the problem of learning a {\em mixture of linear classifiers}, the aim is to learn a collection of hyperplanes  from a sequence of binary responses. Each response is a result of querying with a vector and indicates the side of a randomly chosen hyperplane from the collection the query vector belong to. This model provides a rich representation of heterogeneous data with categorical labels and has only been studied in some special settings. We look at a hitherto  unstudied problem of query complexity upper bound of recovering all the hyperplanes, especially for the case when the hyperplanes are sparse. This setting is a natural generalization of the extreme quantization problem known as 1-bit compressed sensing. Suppose we have a set of $\ell$ unknown $k$-sparse vectors. We can query the set with another vector $\f{a}$, to obtain the sign of the inner product of $\f{a}$ and a randomly chosen vector from the $\ell$-set. How many queries are sufficient to identify all the $\ell$ unknown vectors? This question is significantly more challenging than both the basic 1-bit compressed sensing problem (i.e., $\ell=1$ case) and the analogous regression problem (where the value instead of the sign is provided). We provide rigorous query complexity results (with efficient algorithms) for this problem.
\end{abstract}

\section{Introduction}

One of the first and most basic tasks of machine learning is to train a binary linear classifier. Given a set of  explanatory variables (features) and the binary responses (labels), the objective of this task is to find the hyperplane in the space of features that best separates the variables according to their responses. In this paper, we consider a natural generalization of this problem and model a classification task as a mixture of $\ell$ components.  In this generalization, each response is stochastically generated by picking a hyperplane uniformly from the set of $\ell$ unknown hyperplanes, and then returning the side of that hyperplane the feature vector lies. The goal is to learn all of these $\ell$ hyperplanes as accurately as possible, using the least number of responses.

This can be termed as a mixture of binary linear classifiers \cite{sun2014learning}. Similar mixture of simple machine learning models have been around for at least the last thirty years \cite{de1989mixtures} with mixture of linear regression models being the most studied ones \cite{chen2014convex, huang2013nonparametric, khalili2007variable, shen2019iterative, song2014robust,wang2019convergence, yi2016solving, zhu2004hypothesis}. Models of this type are pretty good function approximators \cite{bishop1998latent,jordan1994hierarchical} and have numerous applications in modeling heterogeneous settings such as   machine translation~\cite{liang2006end}, behavioral health~\cite{deb2000estimates}, medicine~\cite{blackwell2006applying}, object recognition~\cite{quattoni2005conditional} etc. While  algorithms for learning the parameters of mixture of linear regressions are solidly grounded (such as tensor decomposition based learning algorithms of  \cite{chaganty2013spectral}), in many of the above applications the labels are discrete categorical data, and therefore a mixture of classifiers is a better model than mixture of regressions. To the best of our knowledge,~\cite{sun2014learning} first rigorously studied a mixture of linear classifiers and provided polynomial time algorithm to approximate the subspace spanned by the component classifiers (hyperplane-normals) as well as a prediction algorithm that given a feature and label, correctly predicts the component used. In this paper we study a related but different problem: the sample complexity of learning {\em all} the component hyperplanes. Our model also differs  from \cite{sun2014learning} where the component responses are `smoothened out'. Here the term {\em sample complexity} is used with a slightly generalized meaning than traditional learning theory - as we explain next, and then switch to the term {\em query complexity} instead. 

Recent works on mixture of sparse linear regressions concentrate on an active query based setting \cite{kris2019sampling, mazumdar2020recovery, yin2018learning}, where one is allowed to design a sample point and query an oracle with that point. The oracle then randomly chooses one of the component models and returns the answer according to that model.
In this paper we adapt exactly this setting for binary classifiers. We assume while queried with a point (vector), an oracle randomly chooses one of the $\ell$ binary classifiers, and then returns an answer according to what was chosen. For the most of this paper we concentrate on recovering `sparse' linear classifiers, which implies that each of the classifiers uses only few of the explanatory variables. This setting is in spirit of the well-studied {\em 1-bit compressed sensing} (1bCS) problem. 

\paragraph{1-bit compressed sensing.}  

In 1-bit compressed sensing, linear measurements of a sparse vector are quantized to only 1 bit, e.g. indicating whether the measurement outcome is positive or not, and the task is to recover the vector up to a prescribed Euclidean error with a minimum number of measurements. An overwhelming majority of the literature focuses on the nonadaptive setting for the problem \cite{acharya2017improved, ai2014one, flodin2019superset, gopi2013one, JLBB13, plan2013one}. Also, a large portion of the literature concentrates on learning only the support of the sparse vector from the 1-bit measurements \cite{acharya2017improved, gopi2013one}.
 
It was shown in \cite{JLBB13} that $O(\frac{k}{\epsilon} \log (\frac n\epsilon))$ Gaussian queries\footnote{all coordinates of the query vector are sampled independently from the standard Gaussian distribution} suffice to approximately (to the Euclidean precision $\epsilon$) recover an unknown $k$-sparse vector $\f{\beta}$ using $1$-bit measurements.  Given the labels of the query vectors, one recovers $\f{\beta}$ by finding a $k$-sparse vector that is consistent with all the labels. If we consider enough queries, then the obtained solution is guaranteed to be close to the actual underlying vector. \cite{acharya2017improved} studied a two-step recovery process, where in the first step, they use queries corresponding to the rows of a special matrix, known as {\em Robust Union Free Family (RUFF)}, to recover the support of the unknown vector $\f{\beta}$ and then use this support information to approximately recover $\f{\beta}$ using an additional $\tilde{O}(\frac{k}{\epsilon})$ Gaussian queries.  Although the recovery algorithm works in two steps, the queries are nonadaptive.
\paragraph{Mixture of sparse linear classifiers.}

The main technical difficulty that arises in recovering multiple sparse hyperplanes using $1$-bit measurements (labels) is to \emph{align} the responses of different queries concerning a fixed unknown hyperplane. To understand this better, let us consider the case when $\ell=2$ (see Figure~\ref{fig:l2}). Let $\f{\beta^1}, \f{\beta^2}$ be two unknown $k$-sparse vectors corresponding to two sparse linear classifiers. On each query, the oracle samples a $\f{\beta^i}$, for $i \in \{1,2\}$, uniformly at random and returns the binary label corresponding to it ({\color{blue} $+$} or {\color{red} $-$}). One can query the oracle repeatedly with the same query vector to ensure a response from both the classifiers with overwhelmingly high probability. 


\begin{figure}
	\centering
  \includegraphics[width=0.5\textwidth]{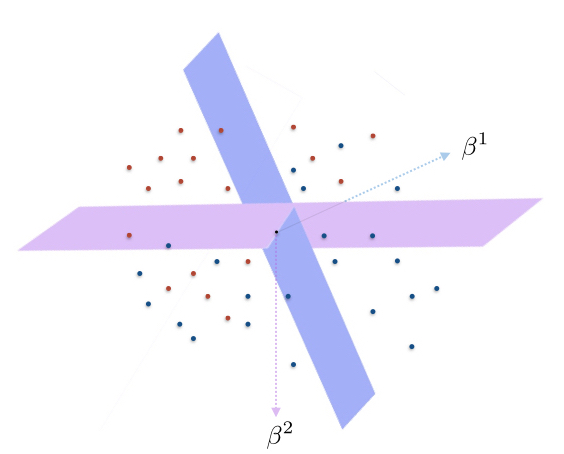}
  \caption{\small{Recover the two lines given red triangles and blue dots. How many such points do we require in order to recover the two lines?}}
  \label{fig:l2}
\end{figure}

For any query vector if the responses corresponding to the two classifiers are the same (\ie, $(+, +)$ or $(-, -)$), then we do not gain any information separating the two classifiers. We might still be able to reconstruct some sparse hyperplanes, but the recovery guarantees of such an algorithm will be poor. 
On the other hand, if both the responses are different (\ie, $(+, -)$), then we do not know which labels correspond to a particular classifier. For example, if the responses are $(+,-)$ and $(+,-)$ for two distinct query vectors, then we do not know if the `plusses' correspond to the same classifier. This issue of alignment makes the problem  challenging.  Such alignment issues are less damning in the case of mixture of linear regressions even in the presence of noise~\cite{kris2019sampling, mazumdar2020recovery, yin2018learning} since we can utilize the magnitude information of the inner products (labels) to our advantage.  

One of the challenges in our study is to recover the supports of the two unknown vectors. 
Consider the case when $\s{supp}(\f{\beta^1}) \neq \s{supp}(\f{\beta^2})$, where $\s{supp}(\f{v})$ denotes the support of the vector $\f{v} \in \bb{R}^n$. In this case, we show that using an RUFF in combination with another similar class of union-free family (UFF), we can deduce the supports of both $\f{\beta^1}$ and $\f{\beta^2}$. Wielding known constructions of such UFFs from literature, we can recover the supports of both $k$-sparse vectors using $O(k^3 \log^2 n )$ queries. 
Once we obtain the supports, we use an additional $O(\frac{k}{\epsilon} \log nk)$ Gaussian queries (with a slight modification) to approximately recover the individual vectors. 

We then extend this two-step process (using more general classes of UFFs) to recover a mixture of $\ell$ different sparse vectors under the assumption that the support of no vector is contained in the union of supports of the remaining ones (Assumption~\ref{assum:sm}). 
The assumption implies that if the sparse vectors are arranged as columns of a matrix, then the matrix contains the identity matrix as a permutation of the rows. This separability condition appears before in~\cite{arora2016computing, donoho2004does, slawski2013matrix} in the context of nonnegative integer matrix factorization, which is a key tool that we will subsequently use to prove our results. To quote \cite{arora2016computing} in the context of matrix factorization, ``an approximate separability condition is regarded as a fairly benign assumption and is believed to hold in many practical contexts in machine learning.'' We believe this observation holds for our context as well (each classifier uses some unique feature). 

We show that with this support separability condition, $\tilde{O}(\ell^6 k^3)$ queries suffice for support recovery of $\ell$ different $k$-sparse vectors. Further, using $\tilde{O}((\ell^3 k/\epsilon))$ queries, we can recover each of the $\f{\beta^i}$'s, for $i \in \{1, \ldots, \ell\}$ up to $\epsilon$ precision (see Theorem~\ref{thm:rec} and Theorem~\ref{thm:twostage}). 

The two-stage procedure described above, can be made completely non-adaptive using queries from union free families  (see Theorem~\ref{thm:onestage}). 

Furthermore, for $\ell=2$, we see that the support condition (Assumption \ref{assum:sm}) is not necessary. We can approximately recover the two unknown vectors provided a) they are not extremely sparse and, b) each $\f{\beta^i} \in \delta \bb{Z}^n$ for some $\delta > 0$. 
To prove this, we borrow the tools from \cite{ai2014one} who give guarantees for 1-bit compressed sensing using sub-Gaussian vectors. In particular, we use queries with independent Bernoulli coordinates which are sub-Gaussian. These discrete random queries (as opposed to continuous Gaussians) along with condition (b), enables us to align the labels corresponding to the two unknown vectors. (see Theorem~\ref{thm:grid} for more details). Note that condition (a) is due to the result by \cite{ai2014one} and is necessary for recovery using sub-Gaussian queries and (b) is a mild assumption on the precision of the unknown vectors, which was also necessary  \cite{kris2019sampling, yin2018learning} for learning the mixture of sparse linear regressions. 

\paragraph{Technical take-away.} As stated above, the main technical hurdle in the paper lies in the support recovery problem for the sparse vectors.  We do it in a few (non-adaptive) steps. First, we try to design queries that will lead us to estimate the size of the intersections of supports for every pair of sparse vectors. If we can design such a set of queries, then using a nonnegative integer matrix factorization techniques we can estimate all the supports. Indeed, this is where the separability assumption comes in handy. Further, because of the separability assumption, it is also possible to design queries from which we can simulate the response of every unknown vector with a random Gaussian query and  tag it. This allows us to recover and {\em cluster} the responses of every unknown vector to a set of random Gaussian queries from which we can approximately recover all the unknown vectors. 

To estimate the size of the intersections of supports for the vector-pairs, we rely on combinatorial designs (and related set-systems), such as pairwise independent cover free families~\cite{furedi1996onr}. While some such union-free families have been used to estimate the support in 1-bit compressed sensing before~\cite{acharya2017improved}, the use of pairwise independent sets to {\em untangle multiple sparse vectors} is new and has almost nothing to do with the recovery of sparse vectors itself.

We leave the problem of designing a query scheme that works for any general $\ell$ without any assumptions as an open problem. Lack of Assumption~\ref{assum:sm} seems to be a fundamental barrier to support recovery as it ensures that a sparse vector will never be in the span of the others.  However, a formal statement of this effect still eludes us.
For large $\ell$, finding out the dependence of query complexity on $\ell$ is also a natural question. Overall, this study leads to an interesting set  of questions that are technically demanding as well as quite relevant to practical  modeling of heterogeneous data that are ubiquitous in applications. 
For instance, in recommendation systems, where the goal is to identify the factors governing the preferences of individual members of a group via crowdsourcing while preserving the anonymity of their responses.

\paragraph{Organization.} The rest of this paper is organized as follows. In the next section, we formally define the problem statement followed by a list of our contributions in Section~\ref{sec:contri}. The notations and the necessary  background on various families of sets are presented in Section~\ref{sec:bckgrnd}. We prove Theorem~\ref{thm:rec} in Section ~\ref{sec:supp-rec}, Theorem~\ref{thm:twostage} in Section~\ref{sec:2stage}, Theorem~\ref{thm:onestage} in Section~\ref{sec:single} and Theorem \ref{thm:grid} in Section \ref{sec:l2} respectively.  Moreover, we demonstrate the capability of our algorithms to learn movie genre preferences of two unknown users using the MovieLens~\cite{harper2015movielens} dataset. The experimental details are included in Section~\ref{sec:exp} of the paper.
\section{Problem Statement}
Let $\f{\beta}^{1},\f{\beta}^{2},\dots,\f{\beta}^{\ell} \in \bb{R}^{n}$ be a set of $\ell$ unknown $k$-sparse vectors. 

Each $\f{\beta^i}$ defines a linear classifier that assigns a label from $\{-1, 1\}$ to every vector in $ \f{v} \in \bb{R}^n$ according to $\s{sign}(\langle \f{v},\f{\beta} \rangle)$, where the sign function $\s{sign}:\bb{R} \rightarrow \{-1, 1\}$ , is defined as,
\[\s{sign}(x) = \begin{cases}
+1 \quad & \text{ if } x \ge 0 \\
-1 \quad &\text{ if }  x< 0 \\
\end{cases}.\]


As described above, our work focuses on recovering the unknown classifiers in the \emph{query model} that was used in  \cite{yin2018learning,kris2019sampling} to study the mixtures of sparse linear regressions. In the query model, we assume the existence of an oracle ${\cal O}$ which when queried with a vector $\f{v} \in \bb{R}^{n}$, samples one of the classifiers $\f{\beta} \in \{ \f{\beta}^{1},\f{\beta}^{2},\dots,\f{\beta}^{\ell}  \}$ uniformly at random and returns the label of $\f{v}$ assigned by the sampled classifier $\f{\beta}$. 
The goal of approximate recovery is to reconstruct each of the unknown classifiers using small number of oracle queries. The problem can be formalized as follows: 
\begin{prob}[$\epsilon$-recovery]
Given $\epsilon > 0$, and query access to oracle $\cal O$, find $k$-sparse vectors $\{ \f{\hat{\beta^1}}, \f{\hat{\beta^2}}, \dots, \f{\hat{\beta^\ell}}\}$ such that for some permutation $\sigma:[\ell]\rightarrow[\ell]$
\begin{align*}
\left\| \frac{\f{\beta^i}}{\| \f{\beta^i} \|_2}-\frac{\f{\hat{\beta}^{\sigma(i)}}}{ \|\f{\hat{\beta}^{\sigma(i)}}\|_2} \right\|_2 \le \epsilon \quad \forall \; i\in [\ell].
\end{align*}
\end{prob}
Since from the classification labels, we lose the magnitude information of the unknown vectors, we assume each $\f{\beta^i}$ and the estimates $\f{\hat{\beta^i}}$ to have a unit norm. 

Similar to the literature on one-bit compressed sensing, one of our proposed solutions employs a two-stage algorithm to recover the unknown vectors. In the first stage the algorithm recovers the support of every vector, and then in the second stage, approximately recovers the vectors using the support information.
 
For any vector $\f{v} \in \bb{R}^n$, let $\s{supp}(\f{v}) := \{i \in [n] \mid \f{v}_i \neq 0 \}$ denote the support of $\f{v}$.  The problem of support recovery is then defined as follows: 
\begin{prob}[Support Recovery]
Given query access to oracle $\cal O$, construct $\{ \f{\hat{\beta^1}},\f{\hat{\beta^2}},\dots, \f{\hat{\beta^\ell}} \}$ such that for some permutation $\sigma:[\ell]\rightarrow[\ell]$
\begin{align*}
\s{supp}(\f{\hat{\beta}^{\sigma(i)}}) = \s{supp}(\f{\beta^{i}}) \quad \forall \; i\in [\ell]
\end{align*}
\end{prob}
For both these problems, we primarily focus on minimizing the query complexity of the problem, \ie, minimizing the number of queries that suffice to approximately recover all the sparse unknown vectors or their supports. However, all the algorithms proposed in this work also run in time $O(\textrm{poly}(q))$, where $q$ is the query complexity of the algorithm.

\section{Our contributions}\label{sec:contri}

In order to present our first set of results, we need certain assumption regarding the separability of supports of the unknown vectors. In particular, we want each component of the mixture to have a unique identifying coordinate. More formally, it can be stated as follows: 
\begin{assumption}{\label{assum:sm}}
For every $i \in [\ell]$, 
$\s{supp}(\f{\beta^i)} \not \subseteq \bigcup_{j: j\neq i} \s{supp} (\f{\beta^j})$, 
i.e. the support of any unknown vector is not contained in the union of the support of the other unknown vectors.
\end{assumption}


\paragraph{Two-stage algorithm:} 
First, we propose a two-stage algorithm for $\epsilon$-recovery of the unknown vectors. In the first stage of the algorithm, we recover the support of the unknown vectors (Theorem~\ref{thm:rec}),  followed by $\epsilon$-recovery using the deduced supports (Theorem~\ref{thm:twostage}) in the second stage. Each stage in itself is non-adaptive, \ie, the queries do not depend on the responses of previously made queries. 

\begin{thm}\label{thm:rec}
Let $\{\f{\beta^1}, \ldots, \f{\beta^\ell} \}$ be a set of $\ell$ unknown $k$-sparse vectors in $\bb{R}^n$ that satisfy Assumption \ref{assum:sm}. 
There exists an algorithm to recover the support of every unknown vector $\{ \f{\beta^i} \}_{i \in [\ell]}$ with probability at least $1-O(1/n^2)$, using  $O(\ell^6 k^3 \log^2 n)$ non-adaptive queries to oracle $\ca{O}$. 
\end{thm}

Now using this support information, we can approximately recover the unknown vectors using an additional $\tilde{O}(\ell^3 k)$ non-adaptive queries. 
\begin{thm}{\label{thm:twostage}}
Let $\{\f{\beta^1}, \ldots, \f{\beta^\ell} \}$ be a set of $\ell$ unknown $k$-sparse vectors in $\bb{R}^n$ that satisfy Assumption \ref{assum:sm}. 
There exists a two-stage algorithm that uses 
$O\Big(\ell^6 k^3 \log^2 n + (\ell^3 k/\epsilon)\log (nk/\epsilon) \log(k/\epsilon) \Big)$ 
 oracle queries for the $\epsilon$-recovery of all the unknown vectors with probability at least $1-O(1/n)$.
\end{thm}
\begin{rmk}
We note that for the two-stage recovery algorithm to be efficient, we require the magnitude of non-zero entries of the unknown vectors to be non-negligible (at least $1/\textrm{exp}(n)$). This assumption however is not required to bound the query complexity of the algorithm which is the main focus of this work. 
\end{rmk}


\paragraph{Completely non-adaptive algorithm: } 
Next, we show that the entire $\epsilon$-recovery algorithm can be made non-adaptive (single-stage) at the cost of increased query complexity. 
\begin{thm}{\label{thm:onestage}}
Let $\{\f{\beta^1}, \ldots, \f{\beta^\ell} \}$ be a set of $\ell$ unknown $k$-sparse vectors in $\bb{R}^n$ that satisfy Assumption \ref{assum:sm}. 
There exists an algorithm that uses 
$O\Big((\ell^{\ell+3} k^{\ell+2}/\epsilon )\log n \log (n/\epsilon)  \log(k/\epsilon) \Big)$ non-adaptive oracle queries for the $\epsilon$-recovery of all the unknown vectors with probability at least $1-O(1/n)$. 
\end{thm}
Note that even though the one-stage algorithm uses many more queries than the two-stage algorithm, a completely non-adaptive is highly parallelizable as one can choose all the query vectors in advance. Also, in the $\ell = O(1)$ regime, the query complexity is comparable to its  two-stage analogue. 

While we mainly focus on minimizing the query complexity, all the algorithms proposed in this work run in $\s{poly}(n)$ time assuming every oracle query takes $\s{poly}(n)$ time and $\ell = o(\log n)$. 

\paragraph{Non-adaptive algorithm for $\ell=2$ without Assumption~\ref{assum:sm}: } For $\ell=2$, we do not need the separability condition (Assumption~\ref{assum:sm}) required earlier for support recovery. Even for $\epsilon$-recovery, instead of Assumption~\ref{assum:sm}, we just need a mild assumption on the precision $\delta$, and the sparsity of the unknown vectors.  In particular, we propose an algorithm for the $\epsilon$-recovery of the two unknown vectors using $\tilde{O}(k^3 + k/\epsilon)$ queries provided the unknown vectors have some finite precision and are not extremely sparse.


\begin{assumption}{\label{assum:max}}
For $\f{\beta} \in \{\f{\beta^1},\f{\beta^2}\}$, $\| \f{\beta} \|_{\infty} = o(1)$.
\end{assumption}

Assumption~\ref{assum:max} ensures that we can safely invoke the result of \cite{ai2014one} who use the exact same assumption in the context of $1$-bit compressed sensing using sub-Gaussian queries. 

\begin{thm}{\label{thm:grid}}
Let $\f{\beta^1}, \f{\beta^2}$ be two $k$-sparse vectors in $\bb{R}^n$ that satisfy Assumption~\ref{assum:max}. Let $\delta > 0$ be the largest real such that $\f{\beta^1}, \f{\beta^2} \in \delta \bb{Z}^n$. There exists an algorithm that uses $O(k^3 \log^2 n + (k^2/\epsilon^4 \delta^2) \log^2 (n/k \delta^2))$ (adaptive) oracle queries for the $\epsilon$-recovery of $\f{\beta^1}, \f{\beta^2}$ with probability at least $1-O(1/n)$.  

Moreover, if $\s{supp}(\f{\beta^1}) \neq \s{supp}(\f{\beta^2})$, then there exists a two-stage  algorithm for the $\epsilon$-recovery of the two vectors using only  $O(k^3 \log^2 n + (k/\epsilon) \log (nk/\epsilon) \log(k/\epsilon))$ non-adaptive oracle queries. 
\end{thm}

Also, the $\epsilon$-recovery algorithm proposed for Theorem~\ref{thm:grid} runs in time $\s{poly}(n, 1/\delta)$. 

\paragraph{No sparsity constraint: } We can infact avoid the sparsity constraint altogether for the case of $\ell =2$. Since in this setting, we consider the support of both unknown vectors to include all coordinates, we do not need a support recovery stage. We then get a single stage and therefore completely non-adaptive algorithm for $\epsilon$-recovery of the two unknown vectors. 

\begin{coro}{\label{coro:general}}
Let $\f{\beta^1}, \f{\beta^2}$ be two unknown vectors in $\bb{R}^n$ that satisfy Assumption~\ref{assum:max}. Let $\delta > 0$ be the largest real such that $\f{\beta^1}, \f{\beta^2} \in \delta \bb{Z}^n$. There exists an algorithm that uses $O( (n^2 /\epsilon^4 \delta^2)\log(1/\delta))$ non-adaptive oracle queries for the $\epsilon$-recovery of $\f{\beta^1}, \f{\beta^2}$ with probability at least $1-O(1/n)$.
\end{coro}

\section{{Preliminaries}}\label{sec:bckgrnd}
Let $[n]$ to denote the set $\{1,2,\dots, n\}$. For any vector $\f{v} \in \bb{R}^n$, $\s{supp}(\f{v})$ denotes the support and $\f{v}_i$ denote the $i^{th}$ entry (coordinate) of the vector $\f{v}$. 
We will use $\f{e_i}$ to denote a vector which has \1 only in the $i^{th}$ position and is \texttt{0} everywhere else. We will use the notation $\langle \f{a},\f{b} \rangle$ to denote the inner product between two vectors $\f{a}$ and $\f{b}$ of the same dimension. For a matrix $\f{A} \in \bb{R}^{m \times n}$, let $\f{A_i} \in \bb{R}^n$ be its $i^{th}$ column and $\f{A}[j]$ denote its $j^{th}$ row. 
and let $\f{A}_{i,j}$ be the $(i,j)$-th  entry of $\f{A}$. 
We will denote by $\s{Inf}$ a \textit{very large positive number}. Also, let $\ca{N}(0,1)$ denote the standard normal distribution. 
We will use ${\cal P}_n$ to denote a the set of all $n \times n$ permutation matrices, \ie, the set of all $n \times n$ binary matrices that are obtained by permuting the rows of an $n \times n$ identity matrix (denoted by $\f{I}_n$). 
Let $\s{round}: \bb{R} \rightarrow \bb{Z}$ denote a function that returns the closest integer to a given real input.

Let us further introduce a few definitions that will be used throughout the paper.

\begin{defn}
For a particular entry $i \in [n]$, define $\ca{S}(i)$ to be the set of all unknown vectors whose $i^{th}$ entry is non-zero.  
\begin{align*}
\ca{S}(i) := \{ \f{\beta^j}, j \in [\ell] \mid \f{\beta^j}_i  \neq 0\}
\end{align*}
\end{defn}

\begin{defn}
For a particular query vector $\f{v}$, define $\s{poscount}(\f{v}),\s{negcount}(\f{v})$ and $\s{nzcount}(\f{v})$ to be the number of unknown vectors that assign a positive, negative, and non-zero label to $\f{v}$ respectively. 
\begin{align*}
\s{poscount}(\f{v}) &:= \lvert \{ \f{\beta^j} \mid \langle \f{v}, \f{\beta^j} \rangle > 0, j \in [\ell] \} \rvert \\
\s{negcount}(\f{v}) &:=  \lvert \{ \f{\beta^j} \mid \langle\f{v}, \f{\beta^j} \rangle < 0 , j \in [\ell] \} \rvert\\
\s{nzcount}(\f{v}) &:= \s{poscount}(\f{v})+\s{negcount}(\f{v}) \\
&= \lvert \{ \f{\beta^j} \mid \langle \f{v}, \f{\beta^j} \rangle \neq 0, j \in [\ell] \} \rvert.
\end{align*}
\end{defn}

\begin{defn}[Gaussian query]\label{def:gaussianQ}
A vector $\f{v} \in \bb{R}^n$ is called a Gaussian query vector if each entry $\f{v}_i$ of $\f{v}$ is sampled independently from the standard Normal distribution, $\ca{N}(0,1)$. 
\end{defn}

\subsection{Estimating the counts}

In this section we show how to accurately estimate each of the counts \ie, $\s{poscount}(\f{v})$, $\s{negcount}(\f{v})$ and $\s{nzcount}(\f{v})$ with respect to any query vector $\f{v}$, with high probability (see Algorithm~\ref{algo:1}).

The idea is to simply query the oracle with $\f{v}$ and $-\f{v}$ repeatedly and estimate the counts empirically using the responses of the oracle. Let $T$ denote the number of times a fixed query vector $\f{v}$ is repeatedly queried. We refer to this quantity as the \emph{batchsize}. We now design estimators of each of the counts which equals the real counts with high probability. Let $\bb{E}_{\f{\beta}}(\cdot)$ and $\Pr_{\f{\beta}}(\cdot)$ denote the expectation and the probability respectively when $\f{\beta}$ is chosen uniformly from the set of unknown vectors $\{\f{\beta}^1,\f{\beta}^2, \dots, \f{\beta}^{\ell}\}$.


We must have
\begin{align*}
\bb{E}_{\f{\beta}}[\s{sign}(\langle \f{v}, \f{\beta} \rangle)]&=
    \bb{E}_{\f{\beta}}[\mathds{1}[\langle \f{v},\f{\beta} \rangle \ge 0]] - \bb{E}_{\f{\beta}}[\mathds{1}[\langle \f{v},\f{\beta} \rangle < 0]] \\
    &= \Pr_{\f{\beta}}\Big( \langle \f{v},\f{\beta} \rangle \ge 0 \Big)-\Pr_{\f{\beta}}\Big( \langle \f{v},\f{\beta} \rangle < 0 \Big) \\
    &= \frac{1}{\ell} \cdot \sum_{i=1}^{\ell} \mathds{1}[\langle \f{v},\f{\beta}^i \rangle \ge 0]-\frac{1}{\ell} \cdot \sum_{i=1}^{\ell} \mathds{1}[\langle \f{v},\f{\beta}^i \rangle < 0].
\end{align*}
Notice that $$\mathds{1}[\langle \f{v},\f{\beta}^i \rangle \ge 0]-\mathds{1}[\langle \f{v},\f{\beta}^i \rangle < 0] = \mathds{1}[\langle -\f{v},\f{\beta}^i \rangle \ge 0]-\mathds{1}[\langle -\f{v},\f{\beta}^i \rangle < 0] \quad \text{if} \quad \langle \f{v},\f{\beta}^i \rangle=0$$
and 
$$\mathds{1}[\langle \f{v},\f{\beta}^i \rangle \ge 0]-\mathds{1}[\langle \f{v},\f{\beta}^i \rangle < 0] = \mathds{1}[\langle -\f{v},\f{\beta}^i \rangle < 0]-\mathds{1}[\langle -\f{v},\f{\beta}^i \rangle \ge 0] \quad \text{if} \quad \langle \f{v},\f{\beta}^i \rangle \neq 0.$$
Therefore, we must have 
\begin{align*}
    \frac{\bb{E}_{\f{\beta}}[\s{sign}(\langle \f{v}, \f{\beta} \rangle)+\s{sign}(\langle -\f{v}, \f{\beta} \rangle)]}{2} = \frac{1}{\ell} \cdot \sum_{i=1}^{\ell} \mathds{1}[\langle \f{v},\f{\beta}^i \rangle = 0]  
\end{align*}

Suppose we query the oracle with the pair of query vectors $\f{v},-\f{v}$ repeatedly for $T$ times. Let us denote the the $T$ responses from the oracle $\ca{O}$ by $y_1,y_2,\dots,y_T$ and $z_1,z_2,\dots,z_T$ corresponding to the query vectors $\f{v}$ and $-\f{v}$ respectively. Hence, we design the following estimator (denoted by $\hat{\s{z}}$) to estimate the number of unknown vectors that have zero projection on the query vector $\f{v}$ i.e. $\sum_{i=1}^{\ell} \mathds{1}[\langle \f{v},\f{\beta}^i \rangle = 0]$:
\begin{align*}
 \hat{\s{z}} \triangleq \s{round}\Big(\frac{\ell\sum_{i=1}^{T}y_i+z_i}{2T}\Big) 
\end{align*}
where $\s{round}:\bb{R}\rightarrow\bb{Z}$ denotes a function that returns the closest integer to a given real input. Again, we have 
\begin{align*}
\bb{E}_{\f{\beta}}[\mathds{1}[\s{sign}(\langle \f{v}, \f{\beta} \rangle)=-1]] = \Pr_{\f{\beta}}[\s{sign}(\langle \f{v}, \f{\beta} \rangle < 0 ] = \frac{1}{\ell} \cdot \sum_{i=1}^{\ell} \mathds{1}[\langle \f{v},\f{\beta}^i \rangle < 0]
\end{align*}
and therefore we design the estimator $$\hat{\s{neg}} \triangleq \s{round}\Big(\frac{\ell\sum_{i=1}^{T} \mathds{1}[y_i=-1]}{T}\Big)$$ of $\s{negcount}(\f{v})$. 
Subsequently, let $\hat{\s{nz}} \triangleq \ell - \hat{\s{z}}$ and $\hat{\s{pos}} \triangleq \hat{\s{nz}}-\hat{\s{neg}}$ be the estimators of $\s{nzcount}(\f{v})$ and $\s{poscount}(\f{v})$ respectively.

\begin{lem}\label{lem:batchsize}
For any query vector $\f{v}$, Algorithm \ref{algo:1} with batchsize $T$ provides the correct estimates of $\s{poscount}(\f{v})$, $\s{negcount}(\f{v})$ and $\s{nzcount}(\f{v})$ with probability at least $1-4e^{-T/2\ell^2}$.
\end{lem}

\begin{proof}
The proof of the lemma follows from a simple application of Chernoff bound. 
 
Let $Z=\sum_i \mathds{1}[y^i = -1]$, and therefore $\bb{E}Z = \frac{T\times\s{negcount}}{\ell}$. 

Note that Algorithm~\ref{algo:1} makes a mistake in estimating $\s{negcount}$ only if 
\begin{align*}
|Z- \frac{T\times\s{negcount}}{\ell}| \ge \frac{T}{2\ell}.
\end{align*}
Since the responses in each batch are independent, using Chernoff bound \cite{boucheron2013concentration}, we get an upper bound on the probability that Algorithm~\ref{algo:1} makes a mistake in estimating $\s{negcount}$ as
\begin{align*}
\Pr \Big(|Z - \bb{E}Z| \ge \frac{T}{2\ell} \Big) \le 2e^{-\frac{T}{2\ell^2}}. 
\end{align*}
Similarly, let $Z'=\frac{\sum_i y^i+z^i}{2}$, and therefore $$\bb{E}Z' = \frac{T}{\ell} \cdot \sum_{i=1}^{\ell} \mathds{1}[\langle \f{x},\f{v}^i \rangle = 0]. $$ 
Again, Algorithm~\ref{algo:1} makes a mistake in estimating $\sum_{i=1}^{\ell} \mathds{1}[\langle \f{x},\f{v}^i \rangle = 0]$ only if 
\begin{align*}
\frac{|Z'- \bb{E}Z'|}{T} \ge \frac{1}{2\ell}.
\end{align*}
Using Chernoff bound \cite{boucheron2013concentration} as before, the probability of making a mistake is bounded from above as
\begin{align*}
\Pr \Big(|Z' - \bb{E}Z'| \ge \frac{T}{2\ell} \Big) \le 2e^{-\frac{T}{2\ell^2}}. 
\end{align*}

By taking a union bound, both $\hat{\s{z}},\hat{\s{neg}}$ are computed correctly with probability at least $1-2e^{-\frac{T}{2\ell^2}}$. Finally, computing $\hat{\s{z}},\hat{\s{neg}}$ correctly implies that $\hat{\s{nz}},\hat{\s{pos}}$ are also correct thus proving our claim.

\end{proof}

\begin{algorithm}[h!]
\caption{\textsc{\textsc{Query}}$(\f{v},T)$\label{algo:1}}
\begin{algorithmic}[1]
\REQUIRE Query access to oracle $\ca{O}$.
\FOR{$i=1,2,\dots,T$}
\STATE Query the oracle with vector $\f{v}$ and obtain response $y^{i} \in \{-1, +1\}$.
\STATE Query the oracle with vector $-\f{v}$ and obtain response $z^{i} \in \{-1, +1\}$.
\ENDFOR

\STATE Let $\hat{\s{z}} :=   \s{round}\Big(\frac{\ell\sum_{i=1}^{T}y_i+z_i}{2T}\Big)$.
\STATE Let $\hat{\s{neg}} := \s{round}\Big(\frac{\ell \sum_i \mathds{1}[y^i = -1]}{T}\Big)$
\STATE Let $\hat{\s{nz}}= \ell - \hat{\s{z}}$ and $\hat{\s{pos}}= \hat{\s{nz}} - \hat{\s{neg}}$
\STATE Return $\hat{\s{pos}} , \hat{\s{neg}}, \hat{\s{nz}}$.
\end{algorithmic}
\end{algorithm}


\subsection{Family of sets}
We now review literature on some important families of sets called \emph{union free families} \cite{acharya2017improved} and \emph{cover free families} \cite{kautz1964nonrandom} that found applications in cryptography, group testing and 1-bit compressed sensing. These special families of sets are used crucially in this work to design the query vectors for the support recovery and the $\epsilon$-recovery algorithms. 

\begin{defn}
[Robust Union Free Family $(d,t,\alpha) - \RUFF$] 
Let $d, t$ be integers and $0 \le \alpha \le 1$. A family of sets, $\ca{F}=\{\ca{H}_1,\ca{H}_2,\dots,\ca{H}_n\}$ where each $\ca{H}_i \subseteq [m]$ and $|\ca{H}|=d$ is a $(d,t,\alpha)$-$\s{RUFF}$ if for any set of $t$ indices $T \subset [n], |T| = t$, and any index $j \notin T$, 
\[
\left| \ca{H}_{j} \setminus \left( \bigcup_{i \in T} \ca{H}_{i} \right) \right| > (1-\alpha) d.
\]
\end{defn}
We refer to $n$ as the size of the family of sets, and $m$ to be the alphabet over which the sets are defined. $\RUFF$s were studied earlier in the context of support recovery of 1bCS \cite{acharya2017improved}, and a simple randomized construction of $(d,t,\alpha)$-$\s{RUFF}$ with $m = O(t^2 \log n)$ was proposed by De Wolf \cite{de2012efficient}.

\begin{lem}{\cite{acharya2017improved,de2012efficient}}
{\label{lem:ruffexist}}
Given $n, t$ and $\alpha > 0$, there exists an $(d,t,\alpha)$-$\s{RUFF}$, $\ca{F}$ with $m=O\big((t^2\log n)/\alpha^2)$ and $d=O((t \log n)/\alpha)$.
\end{lem}

$\RUFF$ is a generalization of the family of sets known as the Union Free Familes ($\UFF$) - which are essentially $(d,t,1)$-$\RUFF$. In this work, we require yet another generalization of $\UFF$ known as Cover Free Families ($\CFF$) that are also sometimes referred to as  superimposed codes \cite{d2014bounds}. 

\begin{defn}[Cover Free Family $(r, t)$-$\CFF$]
A family of sets $\ca{F}=\{\ca{H}_1,\ca{H}_2,\dots,\ca{H}_n\}$ where each $\ca{H}_i \subseteq [m]$ is an $(r, t)$-$\CFF$ if for any pair of disjoint sets of indices $T_1, T_2 \subset [n]$ such that $|T_1| = r, |T_2| = t, T_1 \cap T_2 = \emptyset$, 
\[ \left| \bigcap_{i \in T_1} \ca{H}_{i}  \setminus \bigcup_{i  \in T_2} \ca{H}_{i} \right| > 0.
\]  
\end{defn}

Several constructions and bounds on existence of $\CFF$s are known in literature. We state the following lemma regarding the existence of $\CFF$ which can be found in \cite{ruszinko1994upper, furedi1996onr}. 
\begin{lem}\label{lem:cffexist}
For any given integers $r, t$, there exists an $(r,t)$-$\CFF$, $\ca{F}$ of size $n$ with $m=O(t^{r+1}\log n)$.
\end{lem}

\begin{proof}
We give a non-constructive proof for the existence of $(r, t) -\s{\CFF}$ of size $n$ and alphabet $m=O(t^{r+1}\log n)$. Recall that a family of sets $\ca{F}=\{\ca{H}_1,\ca{H}_2,\dots,\ca{H}_n\}$ where each $\ca{H}_i \subseteq [m]$ is an $(r, t) - \s{CFF}$ if the following holds: for all distinct $j_0,j_1,\dots,j_{t+r-1} \in [n]$, it is the case that $$\bigcap_{p \in \{0,1,\dots,r-1\}}\ca{H}_{j_p}  \not \subseteq \bigcup_{q \in \{r,r+1,\dots,t+r-1\}} \ca{H}_{j_q}. $$ 
Since $\s{PUFF}$ is a special case of $(r, t)-\s{CFF}$ for $r=2$, this result holds for $\s{PUFF}$ as well. 

Consider a matrix $\f{G}$ of size $m \times n$ where each entry is generated independently from a $\m{Bernoulli}(p)$ distribution with $p$ as a parameter. Consider a distinct set of $t+r$ indices $j_0,j_1,\dots,j_{t+r-1} \in [n]$. For a particular row
of the matrix $\f{G}$, the event that there exists a \texttt{1} in the indices $j_0,j_1,\dots,j_{r-1}$ and \texttt{0} in the indices $j_{r},j_{r+1},\dots,j_{t+r-1}$ holds with probability $p^{r}(1-p)^{t}$. Therefore, for a fixed row, this event  does not hold with probability $1-p^{r}(1-p)^{t}$ and the probability that for all the rows the event does not hold is $(1-p^{r}(1-p)^{t})^{m}$. Notice that the number of such possible sets of $t+r$ columns is ${n \choose t+r}{t+r \choose r}$. By taking a union bound, the probability ($P_{e}$) that the event does not hold for all the rows for at least one set of $t+r$ indices is   
\begin{align*}
P_e &\le {n \choose t+r}{t+r \choose r}\big(1-p^{r}(1-p)^{t}\big)^{m} \\
\end{align*} 
Since we want to minimize the upper bound, we want to maximize $p^r(1-p)^t$. Substituting $p=\frac{1}{t+1}$, we get that 
\begin{align*}
p^r(1-p)^t = \Big(\frac{t}{t+1}\Big)^{t}\cdot \frac{1}{(t+1)^{r}} > \frac{1}{e(t+1)^{r}}.
\end{align*}
Further, using the fact that ${n \choose t} \le \Big(\frac{en}{t}\Big)^t$, we obtain
\begin{align*}
P_e \le \frac{(en)^{t+r}}{(t+r)^{t}}\Big(1-\frac{1}{e(t+1)^{r}}\Big)^{m} \le \frac{(en)^{t+r}}{(t+r)^{t}}\exp\Big(-\frac{m}{e(t+1)^{r}}\Big) < \alpha
\end{align*}
for some very small number $\eta$. Taking $\log$ on both sides and after some rearrangement, we obtain 
\begin{align*}
m> e(t+1)^{r}\Big((t+r)\log \frac{en}{t+r}+r\log (t+r)+\log \frac1\eta\Big).
\end{align*}
Hence, using $m=O(t^{r+1}\log n)$, the event holds for at least one row for every set of $t+r$ indices with high probability. Therefore, with high probability, the family of sets $\ca{F} = \{\ca{H}_1,\ca{H}_2,\dots,\ca{H}_n\}$ corresponding to the rows of $\f{G}$ is a $(r, t) - \s{CFF}$.
\end{proof}

Note that $(1, t)$-$\CFF$ is exactly a $\UFF$. The $(2, t)$-$\CFF$ is of particular interest to us and will henceforth be referred to as the \emph{pairwise union free family} ($\PUFF$). From Lemma~\ref{lem:cffexist} we know the existence of $\PUFF$ of size $n$ with $m = O(t^3\log n)$. 
\begin{coro}
{\label{cor:puffexist}}
For any given integer $t$, there exists a $(2,t)$-$\CFF$, $\ca{F}$ of size $n$ with $m=O(t^{3}\log n)$.
\end{coro}



\section{Support Recovery }\label{sec:supp-rec}

\subsection{Proof of Theorem~\ref{thm:rec}}

In this section, we present an efficient algorithm to recover the support of all the $\ell$ unknown vectors using a small number of oracle queries. 
The proof of Theorem~\ref{thm:rec} follows from the guarantees of Algorithm \ref{algo:2}. The  proofs of the helper lemmas used in this theorem are deferred to Section~\ref{subsec:supp-rec}.

Consider the support matrix $\f{X} \in \{0,1\}^{n \times \ell}$ where the $i$-th column is the indicator of $\s{supp}(\f{\beta^{i}})$. The goal in Theorem~\ref{thm:rec} is to recover this unknown matrix $\f{X}$ (up to permutations of columns) using a small number of oracle queries. 
In Algorithm~\ref{algo:2}, we recover $\f{X}$ from $\f{XX^T}$, where the latter can be constructed using only the estimates of $\s{nzcount}$ for some specially designed queries. The unknown matrix $\f{X}$ is recovered from the constructed $\f{XX^T}$ by rank factorization with binary constraints. The factorization is efficient and also turns out to be unique (up to permutations of columns) because of the separability assumption (Assumption~\ref{assum:sm}) on the supports of the unknown vectors. 

The main challenge lies in constructing the matrix $\f{XX^T}$ using just the oracle queries.  Recall that for any $i \in [n]$, $\ca{S}(i)$ denotes the set of unknown vectors that have a non-zero entry in the $i$-th coordinate. 
Note that the $i$-th row of $\f{X}$, for any $i \in [n]$, is essentially the indicator of $\ca{S}(i)$. From this observation, it follows that the $(i,j)$-th entry of $\f{XX^T}$ is captured by the term $|\ca{S}(i) \cap \ca{S}(j)|$. 

We observe that the quantity $|\ca{S}(i) \cap \ca{S}(j)|$ can be computed from oracle queries in two steps. First, we use query vectors from an $\RUFF$ with appropriate parameters to compute $|\ca{S}(i)|$ for every $i \in [n]$ (see Algorithm~\ref{algo:Si}). Then, using queries from a $\PUFF$ (Algorithm~\ref{algo:Sint}) to obtain $|\ca{S}(i) \cup \ca{S}(j) |$ for every pair $(i, j)$. To state it formally, 
\begin{lem}{\label{lem:ruff}}
There exists an algorithm to compute $|\ca{S}(i)|$ for each $i \in [n]$ with probability at least $1 - O\left(1/n^2\right)$ using $O(\ell^4 k^2 \log(\ell k n) \log n )$ oracle queries. 
\end{lem}

\begin{lem}\label{lem:puff}
There exists an algorithm to compute $|\ca{S}(i) \cup \ca{S}(j)|$ for every pair $(i, j)$ with probability at least $1 - O\left(1/n^2\right)$ using $O(\ell^6 k^3 \log(\ell k n) \log n)$ oracle queries. 
\end{lem}

By combining these two steps, we can obtain the $(i,j)$-th entry of $XX^T$ as $|\ca{S}(i) \cap \ca{S}(j)| = |\ca{S}(i)| + |\ca{S}(j)| - |\ca{S}(i) \cup \ca{S}(j) |$. 
Equipped with these two Lemmas, we now prove the guarantees of Algorithm~\ref{algo:2} that completes the proof of Theorem~\ref{thm:rec}.

\begin{algorithm}[h!]
\caption{\textsc{Recover--Support} \label{algo:2}}
\begin{algorithmic}[1]
\REQUIRE Query access to oracle $\cal O$.
\REQUIRE Assumption $\ref{assum:sm}$ to be true. 
\STATE Estimate $|S(i)|$ for every $i \in [n]$ using Algorithm~\ref{algo:Si}.
\STATE Estimate $|S(i) \cup S(j)|$ for every $i, j \in [n]$ using Algorithm~\ref{algo:Sint}.
\FOR{every pair $(i,j) \in [n] \times [n]$}
\STATE Set $Z_{i,j} = |\ca{S}(i)| + |\ca{S}(j)| - |\ca{S}(i) \cup \ca{S}(j) |$
\ENDFOR
\STATE Return $\hat{\f{X}} \in \{0,1\}^{n \times \ell}$ such that $\hat{\f{X}}\hat{\f{X}}^{T}=\f{Z}$.
\end{algorithmic}
\end{algorithm}

\begin{proof}[Proof of Theorem \ref{thm:rec}] 
Using Algorithm~\ref{algo:Si} and Algorithm~\ref{algo:Sint}, we compute $|\ca{S}(i) \cap \ca{S}(j)| = |\ca{S}(i)| + |\ca{S}(i)| - |\ca{S}(i) \cup \ca{S}(j) |$ for every pair $(i,j)\in [n]\times[n]$, and hence populate the entries of $Z = XX^T$. To obtain $X$ from $Z$, we perform a rank factorization of $Z$ with a binary constraint on the factors. We now show that Assumption~\ref{assum:sm} ensures that this factorization is unique up to permutations. 

Suppose $\f{Y} \neq \f{X}$ is a binary matrix such that $\f{Y}\f{Y}^{T}=\f{X}\f{X}^{T}$. Therefore, there exists a rotation matrix $\f{R} \in \bb{R}^{\ell \times \ell}$ such that $\f{Y}=\f{X}\f{R}$. From Assumption \ref{assum:sm} we know that there exists an $\ell \times \ell$ submatrix $\tilde{\f{X}}$ of $\f{X}$ that is a permutation matrix. 
For the corresponding submatrix $\tilde{\f{Y}}$ of $\f{Y}$ (obtained by choosing the same subset of rows), it must hold that 
\begin{align*}
\tilde{\f{Y}}\tilde{\f{Y}}^{T} = \tilde{\f{X}}\tilde{\f{X}}^{T} = \f{I}
\end{align*}
where $\f{I}$ is the $\ell \times \ell$ identity matrix. Since $\f{Y}$ has binary entries, $\tilde{\f{Y}}$ must be a permutation matrix as well. This implies that $\f{R}$ is a permutation matrix and a constrained rank factorization can recover $\f{X}$ up to a permutation of columns. Therefore, Algorithm~\ref{algo:2} successfully recovers the support of all the $\ell$ unknown vectors. 

The total number of queries needed by Algorithm~\ref{algo:2} is the sum total of the queries needed by Algorithm~\ref{algo:Si} and Algorithm~\ref{algo:Sint} which is $O(\ell^6 k^3 \log(\ell k n) \log(n))$.

Moreover, since Algorithm~\ref{algo:Si} and Algorithm~\ref{algo:Sint} each succeed with probability at least $1 - O(1/n^2)$. By a union bound, it follows that Algorithm~\ref{algo:2} succeeds with probability at least $1 - O(1/n^2)$. 
\end{proof} 

\subsection{Proofs of Lemma \ref{lem:ruff} and Lemma \ref{lem:puff}}
\label{subsec:supp-rec}

\paragraph{Compute $|\ca{S}(i)|$ using Algorithm~\ref{algo:Si}.} 
First, we show how to compute $|\ca{S}(i)|$ for every index $i \in [n]$. 
Let $\ca{F} =  \{\ca{H}_1, \ca{H}_2, \dots, \ca{H}_n \}$ be a $(d,\ell k, 0.5)$-$\RUFF$ of size $n$ over alphabet $[m]$.   
Construct the binary matrix $\f{A} \in \{0,1\}^{m \times n}$ from $\ca{F}$, as $\f{A}_{i,j} = \1$ if and only if $i \in \ca{H}_j$. 
Each column $j \in [n]$ of $\f{A}$ is essentially the indicator vector of the set $\ca{H}_j$. 
We use the rows of matrix $\f{A}$ as query vectors to compute $|\ca{S}(i)|$ for each $i \in [n]$. For each such query vector $\f{v}$, we compute the $\s{nzcount}(\f{v})$ using Algorithm~\ref{algo:1} with batchsize $T = O(\ell^2 \log \ell k n)$. The large value of $T$ ensures that the estimated $\s{nzcount}$ is correct for all the queries with very high probability.

For every $h \in \{0, \ldots, \ell \}$, let $\f{b}^h \in \{0,1\}^m$ be the indicator of the queries that have $\s{nzcount}$ at least $h$. We show in Lemma~\ref{lem:ruff} that the set of columns of $\f{A}$ that have large intersection with $\f{b}^h$, exactly correspond to the indices $i \in [n]$ that satisfy $|\ca{S}(i)| \ge h$. This allows us to recover $|\ca{S}(i)|$ exactly for each $i \in [n]$. 
\begin{algorithm}[h!]
\caption{\textsc{Compute--}$|\ca{S}(i)|$ \label{algo:Si}}
\begin{algorithmic}[1]
\REQUIRE Construct binary matrix $\f{A} \in \{ 0,1\}^{m \times n}$ from 
$(d,\ell k,0.5)- \RUFF$ of size $n$ over alphabet $[m]$, with $m=c_1\ell^2k^2\log n$ and $d=c_2\ell k\log n$. 
\STATE Initialize $\f{b}^0, \f{b}^1,\f{b}^2,\dots,\f{b}^\ell$ to all zero vectors of dimension $m$.  
\STATE Let batchsize $T = 4 \ell^2 \log mn$. 
\FOR{$i=1,\dots,m$}
\STATE Set $w := \s{nzcount}(\f{A}[i])$  (obtained using Algorithm~\ref{algo:1} with batchsize $T$.)
\FOR{$h=0,1,\dots,w$}
\STATE Set $\f{b}^h_i=1$.
\ENDFOR
\ENDFOR
\FOR{$h=0,1,\dots,\ell$}
\STATE Set $\ca{C}_h=\{i \in [n] \mid |\s{supp}(\f{b}^h)\cap \s{supp}(\f{A}_i)|\ge 0.5 d \}$.
\ENDFOR
\FOR{$i=1,2,\dots,n$}
\STATE Set $|\ca{S}(i)|=h$ if $i \in \{\ca{C}_{h}\setminus \ca{C}_{h+1}\}$ for some $h \in \{0,1,\dots, \ell-1\}$.
\STATE Set $|\ca{S}(i)|=\ell$ if $i \in \ca{C}_{\ell}$
\ENDFOR
\end{algorithmic}
\end{algorithm}

\begin{proof}[Proof of Lemma~\ref{lem:ruff}]
Since $\f{A}$ has $m = O(\ell^2k^2\log n)$ distinct rows, and each row is queried $T = O( \ell^2 \log(mn))$ times, the total query complexity of Algorithm~\ref{algo:Si} is $O(\ell^4 k^2 \log(\ell k n) \log n )$. 

To prove the correctness, we first see that the $\s{nzcount}$ for each query is estimated correctly using Algorithm~\ref{algo:1} with overwhelmingly high probability. 
From Lemma \ref{lem:batchsize} with $T = 4\ell^2 \log(mn)$, it follows that each $\s{nzcount}$ is estimated correctly with probability at least $1 - \frac{1}{mn^2}$. 
Therefore, by taking a union bound over all rows of $\f{A}$, we estimate all the counts accurately with probability at least $1-\frac{1}{n^2}$. 

We now show, using the properties of $\RUFF$, that $|\s{supp}(\f{b}^h)\cap \s{supp}(\f{A}_i)|\ge 0.5 d$ if and only if $|\ca{S}(i)| \ge h$, for any $0\le h \le \ell$. 

Let $i \in [n]$ be an index such that $|\ca{S}(i)| \ge h$, i.e., there exist at least $h$ unknown vectors that have a non-zero entry in their $i^{th}$ coordinate. Also, let $U := \cup_{i \in [\ell]} \s{supp}(\f{\beta}^{i})$ denote the union of supports of all the unknown vectors. Since each unknown vector is $k$-sparse, it follows that $| U | \le \ell k$. 
To show that $|\s{supp}(\f{b}^h)\cap \s{supp}(\f{A}_i)|\ge 0.5 d$, consider the set of rows of $\f{A}$ indexed by $W := \{\s{supp}(\f{A}_i) \setminus \cup_{j \in U \setminus \{i\} } \s{supp}(\f{A}_j)\}$. 
Since $\f{A}$ is a $(d, \ell k, 0.5) - \RUFF$, we know that $|W| \ge 0.5 d$. We now show that $\f{b}^h_t = 1$ for every $t \in W$. 
This follows from the observation that for $t \in W$, and each unknown vector $\f{\beta} \in \ca{S}(i)$, the query 
 $\s{sign}(\langle \f{A}[t],\f{\beta} \rangle) = \s{sign}(\f{\beta}_i) \neq 0$.  
Since $|\ca{S}(i)| \ge h$, we conclude that $\s{nzcount}(\f{A}[t]) \ge h$, and therefore, $\f{b}^h_t = 1$. 

To prove the converse, consider an index $i \in [n]$ such that $|\ca{S}(i)| < h$. Using a similar argument as above, we now show that $|\s{supp}(\f{b}^h)\cap \s{supp}(\f{A}_i)| < 0.5 d$.  Consider the set of rows of $\f{A}$ indexed by $W := \{\s{supp}(\f{A}_i) \setminus \cup_{j \in U \setminus \{i\} } \s{supp}(\f{A}_j)\}$. Now observe that for each $t \in W$, and any unknown vector $\f{\beta} \notin \ca{S}(i)$, the query $\s{sign}(\langle \f{A}[t],\f{\beta} \rangle) = 0$.  Therefore $\s{nzcount}(\f{A}[t]) \le |\ca{S}(i)| < h$, and $\f{b}^h_t =0$ for all $t \in W$. Since $|W| \ge 0.5 d$, it follows that $|\s{supp}(\f{b}^h)\cap \s{supp}(\f{A}_i)| < 0.5 d$.

For any $0 \le h \le \ell$, Algorithm~\ref{algo:Si}. therefore correctly identifies the set of indices $i \in [n]$ such that $|\ca{S}(i)| \ge h$. In particular, the set $C_h: = \{i \in [n] \mid |\ca{S}(i)| \ge h\}$. 
Therefore, the set $\ca{C}_{h} \setminus \ca{C}_{h+1}$ is exactly the set of indices $i \in [n]$ such that $|\ca{S}(i)| = h$. 

\end{proof}

\paragraph{Compute $|\ca{S}(i) \cup \ca{S}(j)|$ using Algorithm~\ref{algo:Sint}. }
In this section we present an algorithm to compute $|\ca{S}(i) \cup \ca{S}(j)|$, for every $i, j \in [n]$, using $|\ca{S}(i)|$ computed in the previous step. 
We will need an $\ell k - \PUFF$ for this purpose. 
Let $\ca{F} = \{\ca{H}_1, \ca{H}_2, \dots, \ca{H}_n \}$ be the required $\ell k - \PUFF$ of size $n$ over alphabet $m' = O(\ell ^3k^3\log n)$. 

Construct a set of $\ell + 1$ matrices $\ca{B} = \{\f{B^{(1)}}, \ldots,  \f{B^{(\ell+1)}}\}$ 
where, each $\f{B^{(w)}} \in \bb{R}^{m' \times n}, w \in [\ell+1]$, is obtained from the $\PUFF$ $\ca{F}$ in the following way: 
For every $(i,j) \in [m'] \times [n]$, set $\f{B}^{(w)}_{i,j}$ to be a random number sampled uniformly from $[0,1]$ if $i \in H_j$, and $0$ otherwise. We remark that the choice of uniform distribution in $[0,1]$ is arbitrary, and any continuous distribution works. 
 
Since every $\f{B^{(w)}}$ is generated identically, they have the exact same support, though the non-zero entries are different. Also, by definition, the support of the columns of every $\f{B^{(w)}}$ corresponds to the sets in $\ca{F}$. 

Let $U := \cup_{i \in [\ell]} \s{supp}(\f{\beta}^{i})$ denote the union of supports of all the unknown vectors. Since each unknown vector is $k$-sparse, it follows that $| U | \le \ell k$. From the properties of $\ell k - \PUFF$, we know that for any pair of indices $(i, j) \in U \times U$, the set 
$(\ca{H}_{i} \cap \ca{H}_j)  \setminus \bigcup_{q \in U \setminus \{ i,j\}} \ca{H}_{q}$
is non-empty. 
This implies that for every $w \in [\ell+1]$, there exists at least one row of $\f{B}^{(w)}$ that  has a non-zero entry in the $i^{th}$ and $j^{th}$ index, and $0$ in all other indices $p \in U \setminus \{ i,j\}$. In Algorithm~\ref{algo:Sint} we use these rows as queries to estimate their $\s{nzcount}$. In Lemma~\ref{lem:puff}, we show that this quantity is exactly $|S(i) \cup S(j)|$ for that particular pair $(i, j) \in U \times U$.  

\begin{algorithm}[h!]
\caption{\textsc{Recover--}$|\ca{S}(i) \cup \ca{S}(j)|$ \label{algo:Sint}}
\begin{algorithmic}[1]
\REQUIRE $|\ca{S}(i)|$  for every $i \in [n]$. 
\REQUIRE For every $w \in [\ell+1]$, construct $\f{B^{(w)}} \in \bb{R}^{m' \times n}$ from $\ell k - \PUFF$ of size $n$ over alphabet $m'=c_3\ell ^3k^3\log n$.
\STATE Let $U := \{i \in [n] \mid |\ca{S}(i) | >0\}$
\STATE Let batchsize $T = 10\ell^2 \log(n m')$
\FOR{every $p \in [m']$}
\STATE Let $\s{count}(p) := \max_{w \in [\ell+1]} \{ \s{nzcount}(\f{B^{(w)}}[p]) \}$ \\(obtained using Algorithm~\ref{algo:1} with batchsize $T$).
\ENDFOR
\FOR{every pair $(i, j) \in [n] \times [n]$}
	\IF{$i==j$}
		\STATE Set $|\ca{S}(i) \cup \ca{S}(j)| = |\ca{S}(i)|$
	\ELSIF{$i \notin U$}
		\STATE Set $|\ca{S}(i) \cup \ca{S}(j)| = |\ca{S}(j)|$
	\ELSIF{$j \notin U$}
		\STATE Set $|\ca{S}(i) \cup \ca{S}(j)| = |\ca{S}(i)|$
	\ELSE
		\STATE Let $p \in [m']$ such that $\f{B^{(1)}_{p,i}} \neq 0$, $\f{B^{(1)}_{p,j}} \neq 0$, and $\f{B^{(1)}_{p,q}} = 0$ for all $q \in U \setminus \{ i,j\}$.
		\STATE Set $|\ca{S}(i) \cup \ca{S}(j)| = \s{count}(p)$.
	\ENDIF
\ENDFOR
\end{algorithmic}
\end{algorithm}

\begin{proof}[Proof of Lemma~\ref{lem:puff}]
Computing each $\s{count}$ requires $O(T \ell)$ queries. Therefore, the total number of oracle queries made by Algorithm~\ref{algo:Sint} is at most $O(m' T \ell) = O(\ell^6 k^3 \log(\ell k n) \log n)$ for $m' = O(\ell ^3k^3\log n)$ and $T = 10 \ell^2 \log(n m')$. 
Also, observe that each $\s{nzcount}$ is estimated correctly with probability at least $1 - O\left(1/ \ell m' n^2 \right)$. Therefore from union bound it follows that all the $(\ell+1)m'$ estimations of $\s{nzcount}$ are correct with probability at least $1 - O\left(1/n^2\right)$.

Recall that the set $U$ denotes the union of supports of all the unknown vectors. 
This set is equivalent to $\{i \in [n] \mid |\ca{S}(i) | > 0 \}$. 
First, note that if $| \ca{S}(i) | = 0$, there are no unknown vectors supported on the $i^{th}$ index. 
Therefore, $|\ca{S}(i) \cup \ca{S}(j)| = |\ca{S}(j)|$. 
Also, if $i=j$, then the computation of $|\ca{S}(i) \cup \ca{S}(j)|$ is trivial.

We now focus on the only non-trivial case when $(i, j) \in U \times U$ and $i \neq j$. 
Since for every $w \in [\ell+1]$, the support of the columns of $\f{B^{(w)}}$ are the indicators of sets in $\ca{F}$, the $\PUFF$ property implies that there exists at least one row (say, with index $p \in [m']$) of every $\f{B}^{(w)}$ which has a non-zero entry in the $i^{th}$ and $j^{th}$ index, and $0$ in all other indices $q \in U \setminus \{ i,j\}$, i.e., 
\[
\f{B^{(w)}_{p,i}} \neq 0, \f{B^{(w)}_{p,j}} \neq 0 \mbox{, and } \f{B^{(w)}_{p,q}} = 0 \mbox{ for all  } q \in U \setminus \{ i,j\}.
\]
To prove the correctness of the algorithm, we need to show the following:
\[
|\ca{S}(i) \cup \ca{S}(j)| = \max_{w \in [\ell+1]} \{ \s{nzcount}(\f{B^{(w)}}[p]) \}
\]
First observe that using the row $\f{B^{(w)}}[p]$ as query will produce non-zero value for only those unknown vectors $\beta \in  \ca{S}(i) \cup \ca{S}(j)$. This establishes the fact that $|\ca{S}(i) \cup \ca{S}(j)| \ge \s{nzcount}(\f{B^{(w)}}[p])$. 
To show the other side of the inequality, consider the set of $(\ell+1)$ $2$-dimensional vectors obtained by the restriction of rows $\f{B^{(w)}}[p]$ to the coordinates $(i, j)$, 
\[
\{ ( \f{B^{(w)}_{p,i}},  \f{B^{(w)}_{p,j}} ) \mid w \in [\ell+1] \}.
\] 
Since these entries are picked uniformly at random from $[0,1]$, they are pairwise linearly independent. 
Therefore, each $\beta \in  \ca{S}(i) \cup \ca{S}(j)$ can have $\s{sign}(\langle \f{B^{(w)}}[p], \f{\beta} \rangle) = 0$ for at most $1$ of the $w$ queries. So by pigeonhole principle, at least one of the query vectors $\f{B^{(w)}}[p]$ will have $\s{sign}(\langle \f{B^{(w)}}[p], \f{\beta} \rangle) \neq 0$ for all $\beta \in  \ca{S}(i) \cup \ca{S}(j)$. Hence, $|\ca{S}(i) \cup \ca{S}(j)| \le \max_w \{ \s{nzcount}(\f{B^{(w)}}[p]) \}$. 

\end{proof}

\section{Two-stage Approximate Recovery (Proof of Theorem~\ref{thm:twostage}) }\label{sec:2stage}
In this section, we present the proof of Theorem~\ref{thm:twostage}. The two stage approximate recovery algorithm, as the name suggests, proceeds in two sequential steps. In the first stage, we recover the support of all the $\ell$ unknown vectors (presented in  Algorithm~\ref{algo:2} in Section~\ref{sec:supp-rec}). In the second stage, we use these deduced supports to approximately recover the unknown vectors (Algorithm~\ref{algo:3} described below. 


Once we have the obtained the support of all unknown vectors, the task of approximate recovery 
 can be achieved using a set of \emph{Gaussian queries}.  
Recall from Definition~\ref{def:gaussianQ}, a Gaussian query refers to an oracle query with vector $\f{v} = (\f{v}_1, \ldots, \f{v}_n) \in \bb{R}^n$ where each $\f{v}_i$ is sampled independently from the standard Normal distribution, $\f{v}_i \sim \ca{N}(0,1)$. 
The use of Gaussian queries in the context of 1-bit compressed sensing ($\ell=1$) was studied by \cite{JLBB13}.  
\begin{lem}[{\cite{JLBB13}}]\label{lem:reco_orig}
For any $\epsilon > 0$, there exists an $\epsilon$-recovery algorithm to efficiently recover an unknown vector in $\bb{R}^n$ using $O\left(\frac{n}{\epsilon}\log\frac n\epsilon\right)$ Gaussian queries. 
\end{lem}

In the current query model however, the approximate recovery is a bit intricate since we do not possess the knowledge of the particular unknown vector that was sampled by the oracle.  
To circumvent this problem, we will leverage the special support structure of the unknown vectors. From Assumption~\ref{assum:sm}, we know that every unknown vector $\f{\beta}^t, t \in [\ell]$, has at least one coordinate which is not contained in the support of the other unknown vectors. 
We will denote the first such coordinate by $\coord(\f{\beta}^t)$. Define,  
\begin{align*}
\coord(\f{\beta}^t) := \s{min}_{p} \{ \quad p \in \s{supp}(\f{\beta}^t) \setminus \bigcup_{q \in [\ell]\setminus \{t\}}\s{supp}(\f{\beta}^q) \} \in [n].
\end{align*}

For $\epsilon$-recovery of a fixed unknown vector $\f{\beta}^t$, we will use the set of representative coordinates $\{ \coord( \f{\beta}^{t'} ) \}_{t' \neq t}$, to correctly identify its  responses with respect to a set of Gaussian queries. 
In order to achieve this, we first have to recover the sign of $\f{\beta}^t_{\coord(\f{\beta}^t)}$ for every $t \in [\ell]$, using an $\RUFF$, which is described in Algorithm~\ref{algo:sign}. 

\begin{lem}\label{lem:signrec}
Algorithm \ref{algo:sign} recovers $\s{sign}(\f{\beta}^t_{\coord(\f{\beta}^t)}) $ for all $t \in [\ell]$.
\end{lem}

With the knowledge of all the supports, and the sign of every representative coordinate, we are now ready to prove Theorem~\ref{thm:twostage}. The details are presented in the Algorithm~\ref{algo:3}.

\begin{algorithm}[h!]
\caption{\textsc{$\epsilon$-\textsc{Recovery}, Two Stage}\label{algo:3}}
\begin{algorithmic}[1]
\REQUIRE Query access to oracle ${\ca O}$.
\REQUIRE Assumption $\ref{assum:sm}$ to be true.
\STATE Estimate $\s{supp}(\f{\beta}^t)$ for all $t \in [\ell]$ using Algorithm~\ref{algo:2}.
\STATE Estimate $\s{sign}(\f{\beta}^t_{\coord(\f{\beta}^t)})$ for all $t \in [\ell]$ using Algorithm~\ref{algo:sign}.
\STATE Let $\s{Inf}$ be a large positive number. 
\STATE Let batchsize $T = 4\ell^2 \log (nk/\epsilon)$.
\FOR{$t = 1, \ldots, \ell$} 
\FOR{$i = 1, \ldots, \tilde{O}(k/\epsilon)$}
	\STATE Define $\f{v}^t_j:= \left\{ \begin{array}{ll}
		\s{Inf}  & \mbox{if } j =  \coord(\f{\beta}^{t'}), \text{for some } t' \neq t \\
		\ca{N}(0,1) & \mbox{otherwise }
	\end{array}\right.$
	\STATE Obtain $\s{poscount}(\f{v}^t)$ using Algorithm~\ref{algo:1} with batchsize $T$. 
	\STATE Let $\sc{p}_t := | \{t' \neq t \mid \s{sign}(\f{\beta}^{t'}_{\coord(\f{\beta}^{t'})}) = +1 \} |$
	\IF{$\s{poscount}(\f{v}^t) \neq \sc{p}_t$ }
		\STATE Set $y^t_i = +1$.
	\ELSE
		\STATE Set $y^t_i = -1$.
	\ENDIF
\ENDFOR
\STATE From $\{y^t_1,y^t_2,\dots,y^t_{\tilde{O}(k/\epsilon)}\}$, and $\s{supp}(\f{\beta}^t)$ recover $\hat{\f{\beta}^t}$ by using Lemma \ref{lem:reco_orig}.
\ENDFOR
\STATE Return $\{ \hat{\f{\beta}^t}, t \in [\ell] \}$.
\end{algorithmic}
\end{algorithm}

\begin{proof}[Proof of Theorem~\ref{thm:twostage}]
For the $\epsilon$-recovery of a fixed unknown vector $\f{\beta}^t, t \in [\ell]$, we will generate its correct response with respect to a set of $\tilde{O}(k/\epsilon)$ Gaussian queries using \emph{modified} Gaussian queries. 
A modified Gaussian query $\f{v}^t$  for the $t$-th unknown vector, is a Gaussian query with a large positive entry in the coordinates indexed by $\coord(\f{\beta}^{t'})$, for every $t' \neq t$. 

Consider a fixed unknown vector  $\f{\beta}^{t}$. 
Let $\f{v} \in \bb{R}^n$ be a Gaussian query, i.e., every entry of $\f{v}$ is sampled independently from $\ca{N}(0,1)$. 
Algorithm~\ref{algo:3} constructs a modified Gaussian query $\f{v}^{t}$ from $\f{v}$ as follows: 
\begin{align*}
\f{v}^{t}_{j} =
\begin{cases}
\s{Inf} \quad & \text{if } \quad j=\coord(\f{\beta}^{t'})  \quad \text{for some } t'\neq t \\
\f{v}_j \quad &\text{otherwise}
\end{cases}.
\end{align*}

From construction, we know that $\f{v}^{t}_j = \f{v}_j$ for all $j \in \s{supp}(\f{\beta}^t)$. Therefore, 
\[
\langle \f{v}^{t}, \f{\beta}^t \rangle =\langle \f{v}, \f{\beta}^t \rangle \quad \quad \text{and therefore} \quad \quad \s{sign}(\langle \f{v}^{t}, \f{\beta}^t \rangle) =\s{sign}(\langle \f{v}, \f{\beta}^t \rangle).
\]
On the other hand, if $\s{Inf}$ is chosen to be large enough, 
\begin{align*}
\s{sign}(\langle \f{v}^t, \f{\beta}^{t'} \rangle) = \s{sign}(\f{\beta}^{t'}_{\coord(\f{\beta}^{t'})}) \quad \quad \forall t'\neq t, 
\end{align*}
since $\s{Inf} \cdot \f{\beta}^{t'}_{\coord(\f{\beta}^{t'})}$ dominates the sign of the inner product. Note that in order to obtain an upper bound on the value of $\s{Inf}$, we have to assume that the non-zero  entries of every unknown vector have some non-negligible magnitude (at least $1/\textrm{poly}(n)$).

Note that the $\s{sign}(\f{\beta}^{t'}_{\coord(\f{\beta}^{t'})})$ was already computed using Algorithm~\ref{algo:sign}, and therefore, the response of the modified Gaussian query with each $\f{\beta}^{t'}, t' \neq t$ is known. 
Now if $\s{poscount}(\f{v}^t)$ is different from the number of positive instances of $ \s{sign}(\f{\beta}^{t'}_{\coord(\f{\beta}^{t'})}), t' \neq t$, then it follows that $\s{sign}(\langle \f{v}^{t}, \f{\beta}^t \rangle) = +1$. From this we can successfully obtain the response of $\f{\beta}^t$ corresponding to a Gaussian query $\f{v}$. 

Algorithm~\ref{algo:3} simulates $O(k/\epsilon \cdot \log(k/\epsilon))$ Gaussian queries for every $\f{\beta}^t, t \in [\ell]$ using the modified Gaussian queries $\f{v}^t$. Approximate recovery is then possible using Lemma~\ref{lem:reco_orig} (restricted to the $k$-non zero coordinates in the $\s{supp}(\f{\beta}^t)$). 

We now argue about the query complexity and the success probability of Algorithm~\ref{algo:3}. 

For every unknown vector $\f{\beta}^t,  t \in [\ell]$, we simulate $O(k/\epsilon \cdot \log( k/\epsilon))$ Gaussian queries. Simulating each Gaussian query involves $T = O(\ell^2 \log (nk/\epsilon))$ oracle queries to estimate the $\s{poscount}$. 
Note that Algorithm~\ref{algo:sign} can be run simultaneously with Algorithm~\ref{algo:Si} since they use the same set of queries. The sign recovery algorithm, therefore, does not increase the query complexity of approximate recovery. The total query complexity of Algorithm~\ref{algo:3} after the support recovery procedure is at most $O\left( (\ell^3 k/\epsilon) \log(nk/\epsilon) \log(k/\epsilon)\right)$.

From Lemma~\ref{lem:batchsize}, each $\s{poscount}$ is correct with probability at least $1 - O(\epsilon/ (n^2k^2))$ and therefore by a union bound over all the $O(\ell k / \epsilon \cdot log (k/\epsilon))$ $\s{poscount}$ estimates, the algorithm succeeds with probability at least $1 - O(1/n)$.
\end{proof}


\begin{proof}[Proof of Lemma~\ref{lem:signrec}]
Consider the $(d, \ell k, 0.5) - \RUFF$, $\ca{F} =  \{\ca{H}_1, \ca{H}_2, \dots, \ca{H}_n \}$, of size $n$ over alphabet $m = O(\ell^2k^2\log n)$ used in Algorithm~\ref{algo:Si}. Let $\f{A} \in \{0,1\}^{m \times n}$ be the binary matrix constructed from the $\RUFF$ in a similar manner, i.e., $\f{A}_{i,j} = 1$ if and only if $i \in \ca{H}_j$.  From the properties of $\RUFF$, we know that for every $t \in [\ell]$, there exists a row (indexed by $i \in [m]$) of $\f{A}$ such that $\f{A}_{i, u(\f{\beta}^t)} \neq 0$, and $\f{A}_{i, j} = 0$ for all $j \in U \setminus \{ u(\f{\beta}^t)\}$, 
where, $U = \cup_{i \in [\ell]} \s{supp}(\f{\beta}^{i})$. Therefore, the query with $\f{A}[i]$ yields non-zero sign with only $\f{\beta}^t$. Since, 
\[
\s{sign}(\langle \f{A}[i], \f{\beta}^t \rangle) = \s{sign}(\langle \f{e}_{u(\f{\beta}^t)}, \f{\beta}^t \rangle) = \s{sign}(\f{\beta}^t_{u(\f{\beta}^t)})
\]
$\s{sign}(\f{\beta}^t_{u(\f{\beta}^t)})$ can be deduced.

\begin{algorithm}[h!]
\caption{\textsc{Compute--}$\s{sign}(\f{\beta}^t_{\coord(\f{\beta}^t)})$ \label{algo:sign}}
\begin{algorithmic}[1]
\REQUIRE Binary matrix $\f{A} \in \{ 0,1\}^{m \times n}$ from 
$(d,\ell k,0.5)- \RUFF$ of size $n$ over alphabet $[m]$, with $m= O(\ell^2k^2\log n)$ and $d= O(\ell k\log n)$. 
\REQUIRE $\coord(\f{\beta}^t) \in [n]$ for all $t \in [\ell]$. 
\STATE Let batchsize $T = 4 \ell^2 \log mn$. 
\STATE Let $U := \cup_{i \in [\ell]} \s{supp}(\f{\beta}^{i})$. 
\FOR{$t =1,\dots, \ell$}
	\STATE Let $i \in  \{\s{supp}(\f{A}_{\coord(\f{\beta}^t)}) \setminus \cup_{j \in U \setminus \{\coord(\f{\beta}^t)\} } \s{supp}(\f{A}_j)\}$
	\IF{$\s{poscount}(\f{A}[i]) > 0$ (obtained using Algorithm~\ref{algo:1} with batchsize $T$.)}
		\STATE $\s{sign}(\f{\beta}^t_{\coord(\f{\beta}^t)}) = +1$.
	\ELSE 
		\STATE $\s{sign}(\f{\beta}^t_{\coord(\f{\beta}^t)}) = -1$.
	\ENDIF
\ENDFOR
\end{algorithmic}
\end{algorithm}

\end{proof}



\newcommand{\D}{\s{D}}

\section{Single stage process for $\epsilon$-recovery (Proof of Theorem~\ref{thm:onestage})}\label{sec:single}

The approximate recovery procedure (Algorithm~\ref{algo:3}), described in Section~\ref{sec:2stage}, crucially utilizes the support information of every unknown vector to design its queries. 
This requirement forces the algorithm to proceed in two sequential stages. 

In particular, Algorithm~\ref{algo:3}, with the knowledge of the support and the representative coordinates of all the unknown vectors, designed modified Gaussian queries that in turn simulated Gaussian queries for a fixed unknown vector. In this section, we achieve this by using the rows of a matrix obtained from an $(\ell, \ell k)-\CFF$. The property of the $\CFF$ allows us to simulate enough Gaussian queries for every unknown vector without the knowledge of their supports.  This observation gives us a completely non-adaptive algorithm for approximate recovery of all the unknown vectors. 

Consider a matrix $\f{A}$ of dimension $m \times n$ constructed from an $(\ell,\ell k) -\CFF$, $\ca{F}=\{\ca{H}_1,\ca{H}_2,\dots,\ca{H}_n\}$ of size $n$ over alphabet $m$, as follows: 
\[
\f{A}_{i,j} =  \begin{cases}
\s{Inf} \quad & \text{if } i \in \ca{H}_j \\
v \sim \ca{N}(0,1) \quad &\text{otherwise}
\end{cases}.
\]
In Lemma~\ref{lem:gaussian_cff}, we show that for every unknown vector $\f{\beta}^t$, there exists a row of $\f{A}$ that simulates the Gaussian query for it. Therefore, using $\tilde{O}(k/\epsilon)$ independent  blocks of such queries will ensure sufficient Gaussian queries for every unknown vector which then allows us to approximately recover these vectors. 


Recall the definition of a representative coordinate of an unknown vector $\f{\beta}^t$, 
\begin{align*}
\coord(\f{\beta}^t) := \s{min}_{p} \{ \quad p \in \s{supp}(\f{\beta}^t) \setminus \bigcup_{q \in [\ell]\setminus \{t\}}\s{supp}(\f{\beta}^q) \} \in [n].
\end{align*}

\begin{lem}\label{lem:gaussian_cff}
For every $t \in [\ell]$, there exists at least one row $\f{v}^{t}$ in $\f{A}$  that simulates a Gaussian query for $\f{\beta}^t$, and $\s{sign}(\langle \f{v}^t, \f{\beta}^{t'} \rangle) = \s{sign}(\f{\beta}^{t'}_{\coord(\f{\beta}^{t'})})$ for all  $t'\neq t$. 
\end{lem}

\begin{proof}[Proof of Lemma~\ref{lem:gaussian_cff}]
For any fixed $t \in [\ell]$, consider the set of indices  
\begin{align*}
\ca{X}= \{\coord(\f{\beta}^{t'}) \mid t' \in [\ell]\setminus \{t\}\}. 
\end{align*}
Recall that from the property of $(\ell, \ell k) - \CFF$, we must have 
\begin{align*}
\bigcap_{j \in \ca{X}} \s{supp}(\f{A}_j) 
\not \subseteq \bigcup_{j \in \cup_{q \in [\ell]} \s{supp}(\f{\beta}^q) \setminus \ca{X}} \s{supp}(\f{A}_j).
\end{align*}
Therefore, there must exist at least one row $\f{v}^t$ in $\f{A}$ which has a large positive entry, $\s{Inf}$, in all the coordinates indexed  by $\ca{X}$. Moreover,  $\f{v}^t$ has a random Gaussian entry in all the other coordinates indexed by the union of support of all  unknown vectors. 
Since $\f{\beta}^t$ is \texttt{0} for all coordinates in $\ca{X}$, the query $ \s{sign}(\langle \f{v}^{t}, \f{\beta}^t \rangle)$  simulates a Gaussian query.
Also, 
\begin{align*}
\s{sign}(\langle \f{v}, \f{\beta}^{t'} \rangle) = \s{sign}(\coord(\f{\beta}^{t'})) \quad \quad \forall t'\neq t
\end{align*}
since $\s{Inf} \times \f{\beta}^{t'}_{\coord(\f{\beta}^{t'})}$ dominates the inner product. 
\end{proof}

We are now ready to present the completely non-adaptive algorithm for the approximate recovery of all the unknown vectors. 

\begin{algorithm}[h!]
\caption{\textsc{$\epsilon$}-{\textsc{Recovery, Single Stage}}\label{algo:4}}
\begin{algorithmic}[1]
\REQUIRE Assumption $\ref{assum:sm}$ to be true. 
\REQUIRE Binary matrix $\f{\tilde{A}} \in \{0,1\}^{m \times n}$ from $(\ell, \ell k)-\CFF$ of size $n$ over alphabet $m = O((\ell k)^{\ell + 1} \log n)$.
\STATE Estimate $\s{supp}(\f{\beta}^t)$ and $\s{sign}(\f{\beta}^t_{\coord(\f{\beta}^t)})$ for all $t \in [\ell]$ using Algorithm~\ref{algo:2} and Algorithm~\ref{algo:sign} respectively. 
\STATE Set $\s{Inf}$ to be a large positive number. 
\STATE Set $\D = O(k / \epsilon \cdot \log(k/\epsilon))$.
\STATE Set batchsize $T = 4\ell^2 \log (mnk / \epsilon)$. 
\FOR{$i = 1, \ldots, m$}
	\FOR{$w =1,2,\dots,\D$}
		\STATE Construct query vector $\f{v}$, where $\f{v}_j = \begin{cases}
\s{Inf} \quad & \text{if } \f{\tilde{A}}_{i,j} = 1 \\
\ca{N}(0,1) \quad &\text{otherwise}
\end{cases}. $
		\STATE $\s{Query}\Big(\f{v},T \Big)$ and set $\f{P}_{i,w} = \s{poscount}(\f{v})$. 
	\ENDFOR
\ENDFOR
\FOR{$t=1,\dots,\ell$}
	\STATE Let $\ca{X} := \{\coord(\f{\beta}^{t'}) \mid t' \in [\ell] \setminus {t}\}$ and $U := \cup_q \s{supp}(\f{\beta}^q) $
	\STATE Let $i \in \{ \cap_{j \in \ca{X}} \s{supp}(\f{\tilde{A}}_j) 
\setminus \bigcup_{j \in U \setminus \ca{X}} \s{supp}(\f{\tilde{A}}_j)  \} \subset [m]$. 
	\STATE Let $p:= |\{t' \neq t \mid \s{sign}(\f{\beta}^t_{\coord(\f{\beta}^t)}) = +1 \}|$
	\FOR{$w = 1, \ldots, \D$}
		\IF{$\f{P}_{i, w} \neq p$}
			\STATE Set  $y^t_w = +1$
		\ELSE
			\STATE Set  $y^t_w = -1$
		\ENDIF
	\ENDFOR
	\STATE From $\{y^t_w \mid w \in [\D] \}$ and $\s{supp}(\f{\beta}^t)$ recover $\f{\hat{\beta}}^t$ by using Lemma~\ref{lem:reco_orig}. 
\ENDFOR
\STATE Return $\{ \f{\hat{\beta}}^t \mid t \in [\ell] \}$.
\end{algorithmic}
\end{algorithm}


\begin{proof}[Proof of Theorem \ref{thm:onestage}]
The proof of Theorem~\ref{thm:onestage} follows from the guarantees of Algorithm~\ref{algo:4}. 
The query vectors of Algorithm~\ref{algo:4} can be represented by the rows of the following matrix: 
\begin{align*}
\f{R}=
\begin{bmatrix}
\f{A} \\
\f{\tilde{A}}+\f{B}^{(1)}\\
\f{\tilde{A}}+\f{B}^{(2)}\\
\vdots \\
\f{\tilde{A}}+\f{B}^{(\D)}\\
\end{bmatrix}
\end{align*} 
where, $\D = O(k/\epsilon \cdot \log k/\epsilon)$ and $\f{A}$ is the matrix obtained from the $(d,\ell k,0.5)- \RUFF$ required by Algorithm~\ref{algo:2} and Algorithm~\ref{algo:sign}. 
The matrix $\f{\tilde{A}}$ is obtained from an $(\ell, \ell k) - \CFF$, $\ca{F}=\{\ca{H}_1,\ca{H}_2,\dots,\ca{H}_n\}$ by setting $\tilde{\f{A}}_{i,j} = \s{Inf}$ if $i \in \ca{H}_j$ and $0$ otherwise, and each matrix $\f{B}^{(w)}$ for $w \in [\D]$ is a Gaussian matrix with every entry $\f{B}^{(w)}_{i,j}$ drawn uniformly at random from standard Normal distribution. 


Algorithm~\ref{algo:4} decides all its query vectors at the start and hence is completely non-adaptive. It first invokes Algorithm~\ref{algo:2} and Algorithm~\ref{algo:sign} to recover the support and the sign of the representative coordinate of every unknown vector $\f{\beta}^t$. Now using the queries from the rows of the matrix $\f{R}$, the algorithm generates at least $\D= \tilde{O}(k/\epsilon)$ Gaussian queries for each unknown vector. 

It follows from Lemma~\ref{lem:gaussian_cff} that each matrix $\tilde{\f{A}} + \f{B}^{(w)}$, for $w \in [\D]$, contains at least one Gaussian query for every unknown vector. Therefore, in total, $\f{R}$ contains at least $D = O(k/\epsilon \cdot \log k/\epsilon)$ Gaussian queries for every unknown vector $\f{\beta}^t$. Using the responses of these Gaussian queries, we can then approximately recover every $\f{\beta}^t$ using Lemma~\ref{lem:reco_orig}. 

The total query complexity is therefore the sum of query complexities of support recovery process (which from Theorem~\ref{thm:rec} we know to be at most $O(\ell^6 k^3 \log(n) \log(\ell k n))$), and the total number of queries needed to generate $O(k / \epsilon \cdot \log(k/\epsilon))$ Gaussian queries (which is $m T D$) for each unknown vector. 
Therefore the net query complexity is $O\Big( (\ell^{\ell+3} k^{\ell+2}/\epsilon) \log n \log(k/\epsilon) \log (n/\epsilon)) \Big)$. Each Algorithm~\ref{algo:2}, \ref{algo:sign} and the Gaussian query generation succeed with probability at least $1 - O(1/n)$, therefore from union bound, Algorithm~\ref{algo:4} succeeds with probability at least $1 - O(1/n)$. 
\end{proof}
\section{Relaxing Assumption~\ref{assum:sm} for $\ell=2$ (Proof of Theorem~\ref{thm:grid})}\label{sec:l2}

In this section, we will circumvent the necessity for Assumption~\ref{assum:sm} when there are only two unknown vectors - $\{ \f{\beta}^1, \f{\beta}^2\}$. We present a two-stage algorithm to approximately recover both the unknown vectors. In the first stage, the algorithm recovers the support of both the vectors, and then using the support information it approximately recovers the two vectors. 

We would like to mention that if $\s{supp}(\f{\beta}^1) \neq \s{supp}(\f{\beta}^2)$, we do not need any further assumptions on the unknown vectors for their approximate recovery. However, if the two vectors have the exact same support, then we need to impose some mild assumptions in order to approximately recover the vectors. 


\subsection{Support Recovery}\label{subsec:l2supprec}
In this section, we show that supports of both the unknown vectors can be inferred directly from $\{ |\ca{S}(i)| \}_{i \in [n]}$ and $\{|\ca{S}(i) \cap \ca{S}(j)| \}_{i, j \in [n]}$. These quantities were  computed using Algorithm~\ref{algo:Si} and using Algorithm~\ref{algo:Sint} respectively. Moreover, the guarantees of both these algorithms (shown in Lemma~\ref{lem:ruff}, and Lemma~\ref{lem:puff}) do not require the unknown vectors to satisfy any special assumption. 

\begin{lem}\label{lem:l2supprec}
There exists an algorithm to recover the support of any two $k$-sparse unknown vectors using $O(k^3 \log^2 n)$ oracle queries with probability at least $1- O(1/n^2)$. 
\end{lem}

\begin{proof}[Proof of Lemma~\ref{lem:l2supprec}]
Consider Algorithm~\ref{algo:l2supprec}. The query complexity and success guarantees both follow from Lemma~\ref{lem:ruff} and Lemma~\ref{lem:puff}. We now prove the correctness of Algorithm~\ref{algo:l2supprec}. 

\begin{algorithm}[h!]
\caption{\textsc{Recover--Support $\ell=2$} \label{algo:l2supprec}}
\begin{algorithmic}[1]
\REQUIRE Access to oracle $\cal O$
\STATE Estimate $|S(i)|$ for every $i \in [n]$ using Algorithm~\ref{algo:Si}.
\STATE Estimate $|S(i) \cap S(j)|$ for every $i, j \in [n]$ using Algorithm~\ref{algo:Sint}.
\IF{$|S(i)| \in \{0, 2\}$ for all $i \in [n]$}
	\STATE $\s{supp}(\f{\beta}^1) = \s{supp}(\f{\beta}^2) = \{ i \in [n] | |\ca{S}(i)| \neq 0 \}$.
\ELSE
	\STATE Let $i_0 = \min \{ i | |S(i)|=1 \}$, and let $i_0 \in \s{supp}(\f{\beta}^1)$
	\FOR{$j \in [n] \setminus \{i_0\} $}
		\IF{$|S(j)| = 2$}
			\STATE Add $j$ to $\s{supp}(\f{\beta}^1)$, and $\s{supp}(\f{\beta}^2)$. 
		\ELSIF{$|S(j)| =1$ and  $|\ca{S}(i_0) \cap \ca{S}(j)| = 0 $}
			\STATE Add $j$ to $\s{supp}(\f{\beta}^2)$.
		\ELSIF{$|S(j)| =1$ and $|\ca{S}(i_0) \cap \ca{S}(j)| = 1$ }
			\STATE Add $j$ to $\s{supp}(\f{\beta}^1)$.
		\ENDIF
	\ENDFOR
\ENDIF
\end{algorithmic}
\end{algorithm}

\paragraph{Case 1: ($\s{supp}(\f{\beta}^1) \neq \s{supp}(\f{\beta}^2)$).} 
First note that the set of coordinates, $i \in [n]$ with $|\ca{S}(i)| = 2$ belong to the support of both the unknown vectors. For the remaining indices in $T := \{ i \in [n] | |S(i)| = 1 \}$, we use the following approach to decide the unknown vector whose support they belongs to.  

If $|T| = 1$, then without loss of generality we can assume $i \in \s{supp}(\f{\beta}^1)$. 
Else if $|T| > 1$, we set the smallest index $i_0 \in T$ to be in $\s{supp}(\f{\beta}^1)$. We then use this index as a pivot to figure out all the other indices $j \in T \cap \s{supp}(\f{\beta}^1)$. If both  $i_0$, and $j$ lie in $ \s{supp}(\f{\beta}^1)$, then $|\ca{S}(i_0) \cap \ca{S}(j)| =1$, otherwise $|\ca{S}(i_0) \cap \ca{S}(j)| =0$. So, using Algorithm~\ref{algo:l2supprec}, we can identify the supports of both the unknown vectors. 

\paragraph{Case 2: ($\s{supp}(\f{\beta}^1) = \s{supp}(\f{\beta}^2)$).} In this case, we observe that $|\ca{S}(i)| \in \{2, 0\}$ for all $i \in [n]$. Therefore, both the unknown vectors have the exact same support, and nothing further needs to be done since $\s{supp}(\f{\beta}^1) = \s{supp}(\f{\beta}^2) = \{ i \in [n] | |\ca{S}(i)| \neq 0 \}$. 
\end{proof}


\subsection{Approximate Recovery}\label{subsec:l2approxrec}
In this section, we present the approximate recovery algorithm. The queries are designed based on the supports of the two vectors. 

We split the analysis in two parts. First, we consider the case when the two vectors have different supports, i.e. $\s{supp}(\f{\beta}^1) \neq \s{supp}(\f{\beta}^2)$. In this case, we use Lemma~\ref{lem:l2approxrec:case1} to approximately recover the two vectors.

\begin{lem}\label{lem:l2approxrec:case1}
If $\s{supp}(\f{\beta}^1) \neq \s{supp}(\f{\beta}^2)$, then there exists an algorithm for $\epsilon$-approximate recovery of any two $k$-sparse unknown vectors using $O\Big(\frac{k}{\epsilon} \cdot \log(\frac{nk}{\epsilon})\Big)$ oracle queries with probability at least $1- O(1/n)$. 
\end{lem}

When the two vectors have the exact same support, we use a set of sub-Gaussian queries to recover the two vectors. This is slightly tricky, and our algorithms succeeds under some mild assumption on the two unknown vectors (Assumption~\ref{assum:max}).
\begin{lem}\label{lem:l2approxrec:case2}
If $\s{supp}(\f{\beta}^1) = \s{supp}(\f{\beta}^2)$, then there exists an algorithm for $\epsilon$-approximate recovery of any two $k$-sparse unknown vectors using $O(\frac{k^2}{\epsilon^4 \delta^2} \log^2(\frac{nk}{\delta}))$ oracle queries with probability at least $1- O(1/n)$. 
\end{lem}

\begin{algorithm}[h!]
\caption{\textsc{$\epsilon$-\textsc{Approximate-Recovery}}\label{algo:l2approxrec}}
\begin{algorithmic}[1]
\STATE Estimate $\s{supp}(\f{\beta}^1), \s{supp}(\f{\beta}^2)$ using Algorithm~\ref{algo:l2supprec}.
\IF{$\s{supp}(\f{\beta}^1) \neq \s{supp}(\f{\beta}^2)$}
	\STATE Return $\hat{\f{\beta}^1}, \hat{\f{\beta}^2}$ using Algorithm~\ref{algo:l2approxrec:case1}. 
\ELSE
	\STATE Return $\hat{\f{\beta}^1}, \hat{\f{\beta}^2}$ using Algorithm~\ref{algo:l2approxrec:case2}.
\ENDIF
\end{algorithmic}
\end{algorithm}

\begin{proof}[Proof of Theorem~\ref{thm:grid}]
The guarantees of Algorithm~\ref{algo:l2approxrec} prove Theorem~\ref{thm:grid}.
The total query complexity after support recovery is the maximum of the query complexities of Algorithm~\ref{algo:l2approxrec:case1} and Algorithm~\ref{algo:l2approxrec:case2}, which is $O(\frac{k^2}{\epsilon \delta^2} \log^2(\frac{nk}{\delta}))$.

Moreover from Lemma~\ref{lem:l2approxrec:case1} and Lemma~\ref{lem:l2approxrec:case2}, we know  that both these algorithms succeed with a probability at least $1 - O(1/n)$, therefore, Algorithm~\ref{algo:l2approxrec} is also guaranteed to succeed with probability at least $1 - O(1/n)$.
\end{proof}

We now prove Lemma~\ref{lem:l2approxrec:case1} and Lemma~\ref{lem:l2approxrec:case2}. 

\subsubsection{Case 1: $\s{supp}(\f{\beta}^1) \neq \s{supp}(\f{\beta}^2)$. }

\begin{proof}[Proof of Lemma~\ref{lem:l2approxrec:case1}]
Consider a coordinate $p \in \s{supp}(\f{\beta}^1) ~\Delta~ \s{supp}(\f{\beta}^2)$, where $\Delta$ denotes the symmetric difference of the two support sets. Without loss of generality we can assume $p \in \s{supp}(\f{\beta}^1)$. We first identify the $\s{sign}(\f{\beta}^1_p)$ simply using the query vector $\f{e}_p$. For the sake of simplicity let us assume $\s{sign}(\f{\beta}^1_p) = +1$. 

We use two types of queries to recover the two unknown vectors. The \emph{Type 1} queries are modified Gaussian queries, of the form $\f{v} + \s{Inf} \cdot \f{e}_p$, where $\f{v}$ is a Gaussian query vector. \emph{Type 2} query is the plain Gaussian query $\f{v}$. 

Since $p \in \s{supp}(\f{\beta}^1) \setminus \s{supp}(\f{\beta}^2)$, the Type 1 queries will always have a positive response with the unknown vector $\f{\beta}^1$. Moreover, they will simulate a Gaussian query with $\f{\beta}^2$. Therefore from the responses of the oracle, we can correctly identify the response of $\f{\beta}^2$ with a set of $O(k/\epsilon \cdot \log(k/\epsilon))$ Gaussian queries. Now, using Lemma~\ref{lem:reco_orig}, we can approximately recover it. 

Now since the response of $\f{\beta}^2$ with the Type 1 query $\f{v} + \s{Inf} \cdot \f{e}_p$ and  the corresponding Type 2 query $\f{v}$, remains the same, we can also obtain correct responses of $\f{\beta}^1$ with a set of $O(k/\epsilon \cdot \log(k/\epsilon))$ Gaussian queries. By invoking Lemma~\ref{lem:reco_orig} again, we can approximately recover $\f{\beta}^1$. 


\begin{algorithm}[h!]
\caption{\textsc{$\epsilon$-\textsc{Approximate-Recovery: Case 1}}\label{algo:l2approxrec:case1}}
\begin{algorithmic}[1]
\REQUIRE $\s{supp}(\f{\beta}^1) \neq \s{supp}(\f{\beta}^2)$
\STATE Set $m  = O(k/\epsilon \cdot \log(k/\epsilon))$
\STATE Set batchsize $T = 10\log mn$.
\STATE Let $\s{Inf}$ be a large positive number. 
\STATE Let $p \in \s{supp}(\f{\beta}^1) \setminus \s{supp}(\f{\beta}^2)$, and $s:= \s{sign}(\f{\beta}^1_p)$.
\FOR{$i = 1, \ldots, m$}
	\STATE Construct query vector $\f{v}$, where $\f{v}_j =\ca{N}(0,1)$ for all $j \in [n]$.
	\STATE Construct query vector $\tilde{\f{v}} := \f{v} + s \cdot \s{Inf} \cdot \f{e}_p $
	\STATE $\s{Query}\Big(\f{v},T \Big)$, and $\s{Query}\Big(\tilde{\f{v}},T \Big)$.
	\STATE Set $y^i = \begin{cases}
		+1 \quad & \text{if } \s{poscount}(\tilde{\f{v}}) ==2 \\
		-1 \quad &\text{if } \s{negcount}(\tilde{\f{v}}) == 1\\
		0 \quad &\text{otherwise}
		\end{cases}$
	\STATE Set $z^i = \begin{cases}
		+1 \quad & \text{if } y^i = +1 \text{ and } \s{poscount}(\f{v}) ==2 \\
		-1 \quad &\text{if } y^i = +1 \text{ and } \s{negcount}(\f{v}) == 1\\
		+1 \quad & \text{if } y^i = -1 \text{ and } \s{poscount}(\f{v}) ==1 \\
		-1 \quad &\text{if } y^i = -1 \text{ and } \s{negcount}(\f{v}) == 2\\
		+1 \quad & \text{if } y^i = 0 \text{ and } \s{poscount}(\f{v}) ==1 \\
		-1 \quad &\text{if } y^i = 0 \text{ and } \s{negcount}(\f{v}) == 1\\
		0 \quad &\text{otherwise}
		\end{cases}$
\ENDFOR
\STATE From $\{y^i \mid i \in [m] \}$ and $\s{supp}(\f{\beta}^2)$ recover $\f{\hat{\beta}}^2$ by using Lemma~\ref{lem:reco_orig}.
\STATE From $\{z^i \mid i \in [m] \}$ and $\s{supp}(\f{\beta}^1)$ recover $\f{\hat{\beta}}^1$ by using Lemma~\ref{lem:reco_orig}.

\end{algorithmic}
\end{algorithm}

The total query complexity of the algorithm is $O(kT/\epsilon \cdot \log(k/\epsilon)) = O(k/\epsilon \cdot \log(nk / \epsilon) \cdot \log(k/\epsilon))$. Also, from Lemma~\ref{lem:batchsize}, it follows that each oracle query succeeds with probability at least $1 - O(1/mn)$. Therefore by union bound over all $2m$ queries, the algorithm succeeds with probability at least $1 - O(1/n)$. 
\end{proof} 


\subsubsection{Case 2: $\s{supp}(\f{\beta}^1) = \s{supp}(\f{\beta}^2)$.}
We now propose an algorithm for approximate recovery of the two unknown vectors when their supports are exactly the same. Until now for $\epsilon$-recovery, we were using a representative coordinate to generate enough responses to Gaussian queries. However, when the supports are exactly the same, the same trick does not work. 

For the approximate recovery in this case, we use sub-Gaussian queries instead of Gaussian queries. In particular, we consider queries whose entries are sampled uniformly from $\{-1, 1\}$. The equivalent of Lemma~\ref{lem:reco_orig} proved by  \cite{ai2014one} for sub-Gaussian queries enables us  to achieve similar bounds. 



\begin{lem}[Corollary of Theorem 1.1 of \cite{ai2014one}]\label{lem:pv}
Let $\f{x} \in \bb{S}^{n-1}$ be a $k$-sparse unknown vector of unit norm. 
Let $\f{v_1}, \ldots, \f{v_m}$ be independent random vectors in $\mathbb{R}^n$ whose coordinates are drawn uniformly from $\{-1,1\}$. 
There exists an algorithm that recovers $\hat{\f{x}} \in \bb{S}^{n-1}$  using the $1$-bit sign measurements $\{\s{sign}(\langle \f{v_i}, \f{x} \rangle) \}_{i \in [m]}$, such that with probability at least $1 - 4e^{-\alpha^2}$ (for any $\alpha > 0$), it satisfies
\[
\left\| \f{x} - \hat{\f{x}} \right\|_2^2 \le O\left(\|x\|_{\infty}^{\frac12}   + \frac{1}{2 \sqrt{m}}(\sqrt{k \log (2n/k)} + \alpha)  \right). 
\]
\end{lem}

In particular, for $m = O(\frac{k}{\epsilon^4} \log n)$, we get $O(\epsilon + \|x\|_{\infty}^{\frac12})$ - approximate recovery with probability at least $1 - O(1/n)$. Therefore, if the unknown vectors are not \emph{extremely} sparse (Assumption~\ref{assum:max}), we can get good guarantees on their approximate recovery with sufficient number of sub-Gaussian queries. 


%


The central idea of $\epsilon$-recovery algorithm (Algorithm~\ref{algo:l2approxrec:case2}) is  therefore to identify the responses of a particular unknown vector $\f{\beta}$ with respect to a set of sub-Gaussian queries $\f{v} \sim \{-1, 1\}^n$. Then using Lemma~\ref{lem:pv}, we can approximately reconstruct $\f{\beta}$. 

Let us denote by $\response(\f{v})$, the set of distinct responses of the oracle with a query vector $\f{v}$. Since there are only two unknown vectors, $|\response(\f{v})| \le 2$. If both unknown vectors have the same response with respect to a given query vector $\f{v}$, i.e., $|\response(\f{v})| = 1$ then we can trivially identify the correct responses with respect both the unknown vectors by setting $\s{sign}(\langle \f{v}, \f{\beta}^2 \rangle) = \s{sign}(\langle \f{v}, \f{\beta}^2 \rangle) = \response(\f{v})$. 

However if $|\response(\f{v})| = 2$, we need to identify the correct response with respect to a fixed unknown vector. This \emph{alignment} constitutes the main technical challenge in approximate recovery.  To achieve this,  Algorithm~\ref{algo:l2approxrec:case2} fixes a pivot query say $\f{v}_0$ with $|\response(\f{v}_0)| = 2$, and aligns all the other queries with respect to it by making some additional oracle queries.

Let $W$ denote the set of queries such that $|\response(\f{v})| = 2$. Also, for any pair of query vectors, $\f{v}_1, \f{v}_2 \in W$, we denote by $\alignment_{\f{\beta}}(\f{v}_1, \f{v}_2)$ to be an ordered tuple of responses with respect to the unknown vector $\f{\beta}$. 
\[
\alignment_{\f{\beta}}(\f{v}_1, \f{v}_2) = 
(\s{sign}(\langle \f{v}_1, \f{\beta} \rangle), \s{sign}(\langle \f{v}_2, \f{\beta} \rangle)).
\]

We fix a pivot query $\f{v}_0 \in W$ to be one that satisfies $\response(\f{v}_0) = \{-1, 1\}$. We can assume without loss of generality that there always exists one such query, otherwise all queries $\f{v} \in W$ have $0 \in \response(\f{v})$, and Proposition~\ref{prop:align0x0y} aligns all such responses using $O(\log n)$ additional oracle queries.

\begin{prop}\label{prop:align0x0y}
Suppose for all queries $\f{v} \in W$, $0 \in \response(\f{v})$. 
There exists an algorithm that estimates $\alignment_{\f{\beta}^1}(\f{v}_0, \f{v})$ and $\alignment_{\f{\beta}^2}(\f{v}_0, \f{v})$ for any $\f{v}, \f{v}_0 \in W$ using $O(\log n)$ oracle queries with probability at least $1 - O(1/n)$.
\end{prop}

For a fixed pivot query $\f{v}_0 \in W$ such that $\response(\f{v}_0) = \{-1, 1\}$, Proposition~\ref{prop:alignab0x} and Proposition~\ref{prop:alignabxy} compute $\alignment_{\f{\beta}}(\f{v}_0, \f{v})$ for all queries $\f{v} \in W$ such that $0 \in \response(\f{v})$ and $0 \notin \response(\f{v})$ respectively. 

\begin{prop}\label{prop:alignab0x}
Let $\f{v}_0 \in W$ such that $\response(\f{v}_0) = \{-1, 1\}$. For any query vector $\f{v} \in W$ such that $0 \in \response(\f{v})$, there exists an algorithm that computes $\alignment_{\f{\beta}^1}(\f{v}_0, \f{v})$ and $\alignment_{\f{\beta}^2}(\f{v}_0, \f{v})$ using $O(\log n)$ oracle queries with probability at least $1 - O(1/n)$.
\end{prop}

\begin{prop}\label{prop:alignabxy}
Let $\delta > 0$, be the largest real number such that $\f{\beta}^1, \f{\beta}^2 \in \delta \bb{Z}^n$. Let $\f{v}_0 \in W$ such that $\response(\f{v}_0) = \{-1, 1\}$. For any query vector $\f{v} \in W$ such that $\response(\f{v}) = \{-1, 1\}$, there exists an algorithm that computes $\alignment_{\f{\beta}^1}(\f{v}_0, \f{v})$ and $\alignment_{\f{\beta}^2}(\f{v}_0, \f{v})$ using $O(\frac{k}{\delta^2} \log(\frac{nk}{\delta}))$ oracle queries with probability at least $1 - O(1/n)$.
\end{prop}

Using the alignment process and Lemma~\ref{lem:pv}, we can now approximately recover both the unknown vectors. 

\begin{proof}[Proof of Lemma~\ref{lem:l2approxrec:case2}]
Consider Algorithm~\ref{algo:l2approxrec:case2}, which basically collects enough responses of an unknown vector for a set of sub-Gaussian queries by aligning all responses. 

Without loss of generality, we fix $\f{v}_0$ such that $\response(\f{v}_0) = \{+1, -1\}$, and also enforce that $\s{sign}(\f{v}_0, \f{\beta}^1) = +1$. Now, we align all other responses with respect to $\f{v}_0$. The proof of Lemma~\ref{lem:l2approxrec:case2} then follows from the guarantees of Lemma~\ref{lem:pv}. For $m = O(\frac{k}{\epsilon^4} \log n)$, along with the assumptions that $\|\f{\beta^1}\|_{\infty}, \|\f{\beta^2}\|_{\infty} = o(1)$, the algorithm approximately recovers $\f{\beta^1}, \f{\beta^2}$.

\begin{algorithm}[h!]
\caption{\textsc{$\epsilon$-\textsc{Approximate Recovery: Case 2}}\label{algo:l2approxrec:case2}}
\begin{algorithmic}[1]
\REQUIRE $\s{supp}(\f{\beta}^1) = \s{supp}(\f{\beta}^2)$,  Assumption~\ref{assum:max}. 
\STATE Set $m = O(\frac{k}{\epsilon^4} \log(n)) $
\STATE Set batchsize $T = O(\log mn)$
\FOR{$i = 1, \ldots, m$}
	\STATE Sample query vector $\f{v}$ as: $\f{v}_j = 
	\begin{cases}
		+1 \quad & \text{ w.p. } 1/2  \\
		-1 \quad &\text{ w.p. } 1/2
		\end{cases}$

	\STATE $\s{Query}(\f{v}, T)$, and store $\response(\f{v})$.
	\IF{$|\response(\f{v})| == 1$}
		\STATE Set $y^{\f{v}} = \response(\f{v})$. 
		\STATE Set $z^{\f{v}} = \response(\f{v})$. 
	\ELSE
		\STATE Add $\f{v}$ to $W$.
	\ENDIF
	\STATE Let $\f{v}_0$ be an arbitrary $\f{v} \in W$. 	
	\FOR{every $\f{v} \in W$}
		\STATE Set $(y^{\f{v}_0}, y^{\f{v}} ) = \alignment_{\f{\beta}^1}(\f{v}_0, \f{v})$.
		\STATE Set $(z^{\f{v}_0}, z^{\f{v}} ) = \alignment_{\f{\beta}^2}(\f{v}_0, \f{v})$.
	\ENDFOR
\ENDFOR
\STATE Using $\{ y^{\f{v}} \}_{\f{v}}$, estimate $\hat{\f{\beta}^1}$.
\STATE Using $\{ z^{\f{v}} \}_{\f{v}}$, estimate $\hat{\f{\beta}^2}$.
\end{algorithmic}
\end{algorithm}

The number of queries made by Algorithm~\ref{algo:l2approxrec:case2} is at most $mT$ to generate responses and $O(m \frac{k}{\delta^2} \log(\frac{nk}{\delta}))$ to align all the $m$ responses with respect to a fixed pivot query $\f{v}_0$.  Therefore the total query complexity of Algorithm~\ref{algo:l2approxrec:case2} is $O(\frac{k^2}{\epsilon^4 \delta^2} \log^2(\frac{nk}{\delta}))$.

All parts of the algorithm succeed with probability at least $1 - O(1/n)$, and therefore the algorithm succeeds with probability at least $1 - O(1/n)$.
\end{proof}

Finally, we prove Proposition~\ref{prop:align0x0y}, Proposition~\ref{prop:alignab0x} and Proposition~\ref{prop:alignabxy}. 

\begin{proof}[Proof of Proposition~\ref{prop:align0x0y}]
For the proof of Proposition~\ref{prop:align0x0y}, we simply use the query vector $\f{v}_0+\f{v}$ to reveal whether the \texttt{0}'s in the two response sets correspond to the same unknown vector or different ones. The correctness of Algorithm~\ref{algo:align:case1} follows from the fact that there will be a \texttt{0} in the response set of $\f{v}_0+\f{v}$ if and only if both the \texttt{0}'s correspond to the same unknown vector. 

To obtain the complete response set for the query $\f{v}_0+\f{v}$ with probability at least $1- 1/n$, Algorithm~\ref{algo:align:case1} makes at most $O(\log n)$ queries. 

\begin{algorithm}[h!]
\caption{\textsc{Align Queries, Case 1}\label{algo:align:case1}}
\begin{algorithmic}[1]
\REQUIRE $\f{v}_0, \f{v} \in \{-1, 1\}^n$, $0 \in \response(\f{v}_0) \cap \response(\f{v})$. 
\STATE Set batchsize $T = O(\log n)$.
\STATE $\s{Query}(\f{v}_0 + \f{v}, T)$. 
\IF{ $0 \in \response(\f{v}_0 + \f{v})$}
	\STATE $\alignment_{\f{\beta}^1}(\f{v}_0, \f{v}) = (0, 0)$
	\STATE $\alignment_{\f{\beta}^2}(\f{v}_0, \f{v}) = (\response(\f{v}_0) \setminus \{0\},  \response(\f{v}) \setminus \{0\})$
\ELSE
	\STATE $\alignment_{\f{\beta}^1}(\f{v}_0, \f{v}) = (0, \response(\f{v}) \setminus \{0\})$
	\STATE $\alignment_{\f{\beta}^2}(\f{v}_0, \f{v}) = (\response(\f{v}_0) \setminus \{0\},  0)$
	\ENDIF
\end{algorithmic}
\end{algorithm}

\end{proof}


\begin{proof}[Proof of Proposition~\ref{prop:alignab0x}]
In this case, we observe that the response set corresponding to the query $\s{Inf}\cdot \f{v}+\f{v}_0$ can reveal the correct alignment. To see this, let the response of $\f{v}_0$  and $\f{v}$ be $\{+1,-1\}$ and $\{\s{s},0\}$ respectively for some $\s{s} \in \{\pm1\}$. 
The response set corresponding to $\s{Inf}\cdot \f{v}+\f{v}_0$ will be the set (or multi-set) of the form $\{\s{s},\s{t}\}$. Since we know $\s{s} = \response(\f{v}) \setminus \{0\}$, we can deduce $\s{t}$ from the $\s{poscount}(\s{Inf}\cdot \f{v}+\f{v}_0)$, and $\s{negcount}(\s{Inf}\cdot \f{v}+\f{v}_0)$.  

Now, if $\s{t} = +1$, then $(+1, 0)$ are aligned together (response of the same unknown vector) and $(\s{s},-1)$ are aligned together. Similarly, if $\s{t} = -1$, then $(-1, 0)$ and $(+1, \s{s})$ are aligned together respectively. 

The alignment algorithm is presented in Algorithm~\ref{algo:align:case2}. It makes $O(\log n)$ queries and succeeds with probability at least $1-1/n$. 

\begin{algorithm}[h!]
\caption{\textsc{Align Queries, Case 2}\label{algo:align:case2}}
\begin{algorithmic}[1]
\REQUIRE $\f{v}_0, \f{v} \in \{ -1, 1\}^n$, $0 \in \response(\f{v})$, $\response(\f{v}_0) = \{\pm 1\}$. 
\STATE Set batchsize $T = O(\log n)$.
\STATE Set $\s{Inf}$ to be a large positive number. 
\STATE $\s{Query}( \f{v}_0 +\s{Inf}\cdot \f{v}, T)$. 
\IF{ $\response(\f{v}_0 + \s{Inf}\cdot \f{v}) = \{ \response(\f{v}) \setminus \{0\}, +1\}$}
	\STATE $\alignment_{\f{\beta}^1}(\f{v}_0, \f{v}) = (+1, 0)$
	\STATE $\alignment_{\f{\beta}^2}(\f{v}_0, \f{v}) = (-1,  \response(\f{v}) \setminus \{0\})$
\ELSE
	\STATE $\alignment_{\f{\beta}^1}(\f{v}_0, \f{v}) = (+1, \response(\f{v}) \setminus \{0\})$
	\STATE $\alignment_{\f{\beta}^2}(\f{v}_0, \f{v}) = (-1,  0)$
	\ENDIF
\end{algorithmic}
\end{algorithm}
\end{proof}


\begin{proof}[Proof of Proposition~\ref{prop:alignabxy}]
The objective of Proposition~\ref{prop:alignabxy} is to align the responses of queries $\f{v}_0$ and $\f{v}$ by identifying which among the following two hypotheses is true:

\begin{itemize}
\item $\bb{H}_1:$ The response of the unknown vectors with both the query vectors $\f{v}_0$ and $\f{v}$ is same. Since we fixed the $\s{sign}(\langle \f{v}_0, \f{\beta}^1 \rangle) = 1$, this corresponds to the case when $\alignment_{\f{\beta}^1}(\f{v}_0, \f{v}) = (+1, +1)$ and $\alignment_{\f{\beta}^1}(\f{v}_0, \f{v}) = (-1, -1)$. 

In this case, we observe that for any query of the form $\eta \f{v}_0+\zeta \f{v}$ with $\eta, \zeta>0$, the response set will remain $\{+1,-1\}$. 

\item $\bb{H}_2:$ The response of each unknown vector with both the query vectors $\f{v}_0$ and $\f{v}$ is different, i.e., $\alignment_{\f{\beta}^1}(\f{v}_0, \f{v}) = (+1, -1)$ and $\alignment_{\f{\beta}^1}(\f{v}_0, \f{v}) = (-1, +1)$. 

In this case, we note that the response for the queries of the form  $\eta \f{v}_0+\zeta \f{v}$ changes from $\{ -1, 1\}$ to either $\{+1\}, \{-1\}$, or $\{0\}$ for an appropriate choice of $\eta, \zeta > 0$. 
In particular, the cardinality of the response set for queries of the form $\eta \f{v}_0+\zeta \f{v}$ changes from $2$ to $1$ if 
$\frac{\eta}{\zeta} \in \left[-\frac{\langle \f{\beta}^1, \f{v} \rangle}{\langle \f{\beta}^1, \f{v}_0 \rangle}, -\frac{\langle \f{\beta}^2, \f{v} \rangle}{\langle \f{\beta}^2, \f{v}_0 \rangle} \right]  \cup  \left[ -\frac{\langle \f{\beta}^2, \f{v} \rangle}{\langle \f{\beta}^2, \f{v}_0 \rangle}, -\frac{\langle \f{\beta}^1, \f{v} \rangle}{\langle \f{\beta}^1, \f{v}_0 \rangle} \right]$. 
\end{itemize}

\begin{algorithm}[h!]
\caption{\textsc{Align Queries, Case 3}\label{algo:align:case3}}
\begin{algorithmic}[1]
\REQUIRE $\f{v}_0, \f{v} \in \{ 0,-1, 1\}^n$, $\response(\f{v}) = \response(\f{v}_0) = \{\pm 1\}$. 
\STATE Set batchsize $T = O(\log nk/\delta)$.
\FOR{$\eta \in \{\frac{c}{d} \mid c, d \in \bb{Z}\setminus \{0\}, |c|,|d| \le \frac{\sqrt{k}}{\delta}\}$}
	\STATE $\s{Query}(\eta \f{v}_0+\f{v}, T)$. 
	\IF{$|\response(\eta \f{v}_0+\f{v})| ==1$ }
		\STATE Return $\alignment_{\f{\beta}^1}(\f{v}_0, \f{v}) = (+1, -1)$,  $\alignment_{\f{\beta}^2}(\f{v}_0, \f{v}) = (-1, +1)$
	\ENDIF
\ENDFOR
\STATE Return  $\alignment_{\f{\beta}^1}(\f{v}_0, \f{v}) = (+1, +1)$, $\alignment_{\f{\beta}^2}(\f{v}_0, \f{v}) = (-1, -1)$
\end{algorithmic}
\end{algorithm}

In order to distinguish between these two hypotheses, Algorithm~\ref{algo:align:case3} makes sufficient queries of the form $\eta \f{v}_0+\zeta \f{v}$ for varying values of $\eta, \zeta >0$. If for some $\eta, \zeta$ the cardinality of the response set changes from $2$ to $1$, then we claim that $\bb{H}_2$ holds, otherwise $\bb{H}_1$ is true. Algorithm~\ref{algo:align:case3} then returns the appropriate alignment.

Note that for any query vector $\f{v} \in \{-1, 1\}^n$, and any $k$-sparse unknown vector $\f{\beta} \in \bb{S}^{n-1}$ the inner product $\langle \f{\beta}, \f{v} \rangle \in [-\sqrt{k}, \sqrt{k}]$. Moreover, if we assume that the unknown vectors have precision $\delta$, the ratio $\frac{\langle \f{\beta}^2, \f{v} \rangle}{\langle \f{\beta}^2, \f{v}_0 \rangle}$ can assume at most $4k/\delta^2$ distinct values. Algorithm~\ref{algo:align:case3} therefore iterates through all such possible values of $\eta/\zeta$ in order to decide which among the two hypothesis is true. 

The total number of queries made by Algorithm~\ref{algo:align:case3} is therefore $4k T/\delta^2 = O(\frac{k}{\delta^2} \log(\frac{nk}{\delta}))$. From Lemma~\ref{lem:batchsize}, all the responses are recovered correctly with probability $1 - O(1/n)$.
\end{proof}

\section{Experiments}\label{sec:exp}
Similar to the mixed regression model, the problem of learning mixed linear classifiers can be used to model heterogenous data with categorical labels. We provide some simulation results to show the efficacy of our proposed algorithms to reconstruct the component classifiers in the mixture. 

Moreover, the algorithm suggested in this work can be used to learn the set of discriminative features of a group of people in a crowd sourcing model using simple queries with binary responses. Each person's preferences represents a sparse linear classifier, and the oracle queries here correspond to the crowdsourcing model. To exemplify this, we provide experimental results using the MovieLens~\cite{harper2015movielens} dataset to recover the movie genre preferences of two different users (that may use the same account, thus generating mixed responses) using small number of queries.

\subsection{Simulations}
We perform simulations that recover the support of $\ell=2$, $k$-sparse vectors in $\mathbb{R}^n$ using Algorithm~\ref{algo:l2supprec}. We use random sparse matrices with sufficient number of rows to construct an $\s{RUFF}$. Error is measured in terms of relative hamming distance between the actual and the reconstructed support vectors. 

The  simulations show an improvement in the accuracy with increasing number of rows allocated to construct the $\s{RUFF}$ for different values of $n = 1000, 2000, 3000$ with fixed $k=5$. This is evident since the increasing number of rows improve the probability of getting an $\s{RUFF}$. 

\begin{figure}[h!]
	\centering
  \includegraphics[width=0.5\textwidth]{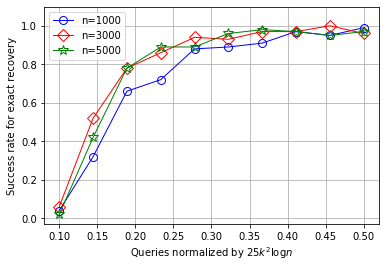}
  \caption{\small{Support Recovery for $\ell=2, k = 5$ and $n = 1000, 2000, 3000$.}}
  \label{fig:sim1}
\end{figure}

\subsection{Movie Lens} 
The MovieLens \cite{harper2015movielens} database contains the user ratings for movies across various genres. Our goal in this set of experiments is to learn the movie genre preferences of two $(\ell=2)$ unknown users using a small set of commonly rated movies.

We first preprocess the set of all movies from the dataset to obtain a subset that have an average rating between 2.5 to 3.5. This is done to avoid biased data points that correspond to movies that are liked (or not liked at all) by almost everyone. For the rest of the experiment, we work with this pre-processed set of movies. 

We consider $n=20$ movie genres in some arbitrary, but predetermined order. The genre preference of each user $i$ is depicted as an (unknown) indicator vector $\f{\beta}^{i} \in \{0,1\}^n$, i.e., $\f{\beta}^i_j = 1$ if and only if user $i$ likes the movies in genre $j$. 
We assume that a user \emph{likes} a particular movie if they rate it $3$ or above. Also, we assume that the user likes a genre if they like at least half the movies they rated in a particular genre.

We consider two users, say $U_1, U_2$ who have commonly rated at least $500$ movies. The preference vectors for both the users is obtained using Algorithm~\ref{algo:l2supprec}. We query the \emph{oracle} with a movie, and obtain its rating from one of the two users at random. For the algorithm, we consider each query to correspond to the indicator of genres that the queried movie belongs to. Using small number of such randomly chosen movie queries, we show that Algorithm~\ref{algo:l2supprec} approximately recovers the movie genre preference of both the users. 

First, we pick a random subset of $m$ movies that were rated by both the users, and partition them into two subsets of size $m_1$, and $m_2$ respectively. The first set of $m_1$ movies are used to partition the list of genres into three classes - genres liked by exactly one of the users, genres liked by both the users, and the genres liked by neither user.  These set of $m_1$  randomly chosen movies essentially correspond to the rows of a $\s{RUFF}$ used in Algorithm~\ref{algo:l2supprec}. 

We then align the genres liked by exactly one of the users, we use the other set of $m_2$ randomly chosen movies and obtain two genre preference vectors $\f{s_1}, \f{s_2}$. 
Since we do not know whether $\f{s_1}$ corresponds the preference vector of $U_1$ or $U_2$, we validate it against both,  i.e., we validate $\f{s_1}$ with $U_1$, $\f{s_2}$ with $U_2$ and vice versa and select the permutation with higher average accuracy. 

\paragraph{Validation: } In order to validate our results, we use our recovered preference vectors to predict the movies that $U_1$ and $U_2$ will like. For each user $U_i$, we select the set of movies that were rated by $U_i$, but were not selected in the set of $m$ movies used to recover their preference vector. The accuracy of our recovered preference vectors are measured by correctly predicting whether a user will like a particular movie from the test set. 

\paragraph{Results: } We obtain the accuracy, precision and recall for three random user pairs who have together rated at least 500 movies.  The results show that our algorithm predicts the movie genre preferences of the user pair with high accuracy even with small $m$. Each of the quantities are obtained by averaging over $100$ runs.
\begin{center}
 \begin{tabular}{|| c || c c | r r r | r r r||} 
 \hline
 id: $(U_1, U_2)$& $m_1$ & $m_2$ & $\s{A}(U_1)$ & $\s{P}(U_1)$& $\s{R}(U_1)$ & $\s{A}(U_2)$ &$\s{P}(U_2)$& $\s{R}(U_2)$  \\ 
 \hline \hline
 & $0$ & $0$ & $0.300$ & $0.000$ & $0.000$ & $0.435$ & $0.000$ &$0.000$ \\ 
 \hline
 (\texttt{68}, \texttt{448}) & $10$ & $20$ & $0.670$ & $0.704$ & $0.916$ & $0.528$ &  $0.550$ &  $0.706$ \\ 
 \hline
& $30$ & $60$ & $0.678$ & $0.700$ & $0.944$ & $0.533$ & $0.548$ & $0.791$ \\
 \hline \hline
 & $0$ & $0$ & $0.269$ & $0.000$ & $0.000$ & $0.107$ & $0.000$ &$0.000$ \\ 
 \hline
 (\texttt{274}, \texttt{380}) & $10$ & $20$ & $0.686$ & $0.733$ & $0.902$ & $0.851$ & $0.893$ &  $0.946$ \\  
 \hline
& $30$ & $60$ & $0.729$ & $0.737$ & $0.982$ & $0.872$ & $0.891$ & $0.976$ \\
 \hline \hline
 & $0$ & $0$ & $0.250$ & $0.000$ & $0.000$ & $0.197$ & $0.000$ &$0.000$ \\ 
 \hline
 (\texttt{474}, \texttt{606}) & $10$ & $20$ & $0.665$ & $0.752$ & $0.827$ & $0.762$ & $0.804$ & $0.930$ \\  
 \hline
& $30$ & $60$ & $0.703$ & $0.750$ & $0.910$ & $0.787$ & $0.806$ & $0.970$ \\
 \hline \hline
\end{tabular}
\end{center}

\section{Conclusion and Open Questions}

In this work, we initiated the study of recovering a mixture of $\ell$ different sparse linear classifiers given query access to an oracle. The problem generalizes the well-studied work on $1$-bit compressed sensing ($\ell = 1$) and also complements the  literature 
on learning mixtures of sparse linear regression in a similar query model. 

Our results for $\ell > 2$, rely on the assumption that the supports of all the unknown vectors are separable. This separability assumption translates to each classifier using a unique feature not being used in others, which happen often in practice. 
The approximate recovery problem without the separability assumption is non-trivial even for $\ell = 2$ case, for which we 
 provide guarantees with much milder assumptions on the  precision of the classifiers.  
 We leave the problem of support recovery and $\epsilon$-recovery without any assumptions as an open problem. 

We primarily focus on providing upper bounds on the query complexity of the support recovery and approximate recovery of the unknown vectors. However, proving optimality results for any such recovery is an interesting open direction. 
It is known that even to recover the support of a single $k$-sparse vector in the 1-bit compressed sensing setting, about $\Omega(k^2 \log n)$ queries are required. This corresponds to $\ell=1$ case, and the lower bound holds trivially for any general $\ell$ as well. 
However, a nontrivial lower bound on the query complexity characterizing the asymptotic dependence on $\ell$, the number of components, will be of interest.

%


\section*{Broader Impact}

This paper is a theoretical study that brings together two seemingly disjoint but equally impactful fields of sparse  recovery and mixture models: the first having numerous applications in signal processing while the second being the main statistical model for clustering. Given that, this work belongs to the foundational area of data science and enhances our understanding of some basic theoretical questions. We feel the methodology developed in this paper is instructive, and exemplifies the use of several combinatorial objects and techniques in signal recovery and classification, that are hitherto underused. Therefore we foresee the technical content of this paper to form good teaching material in foundational data science and signal processing courses. The content of this paper can raise interest  of students or young researchers in discrete mathematics to applications areas and problems of signal processing and machine learning.

While primarily of theoretical interest, the results of the paper can be immediately applicable to some real-life scenarios and be useful in recommendation systems, one of the major drivers of data science research. In particular, if in any case of feedback/rating from users of a service there is ambiguity about the source of the feedback, our framework can be used. This is also applicable to crowdsourcing applications.

{\em Acknowledgements:} This research is supported in part by NSF CCF 1909046 and NSF 1934846.

\bibliographystyle{plain}


\end{document}